\documentclass{article}

\usepackage{amsmath,amsfonts,bm}

\def\eqref#1{equation~\ref{#1}}

\def\floor#1{\lfloor #1 \rfloor}
\def\1{\bm{1}}

\def\rs{{\textnormal{s}}}

\DeclareMathAlphabet{\mathsfit}{\encodingdefault}{\sfdefault}{m}{sl}
\SetMathAlphabet{\mathsfit}{bold}{\encodingdefault}{\sfdefault}{bx}{n}

\newcommand{\R}{\mathbb{R}}

\newcommand{\m}{\hspace{0.25mm}}

\newcommand{\D}{\mathcal{D}}

\usepackage{hyperref}
\usepackage{url}
\usepackage{microtype}
\usepackage{graphicx}
\usepackage{booktabs}       
\usepackage{arydshln}
\usepackage{amssymb}
\usepackage{cuted}
\usepackage{flushend}
\usepackage{graphicx}
\usepackage{subcaption}
\usepackage{rotating}
\usepackage{ifthen}
\usepackage{mathtools}
\usepackage{afterpage}
\graphicspath{ {Images/} }

\newtheorem{theorem}{Theorem}[section]

\newtheorem{proof}[theorem]{Proof}

\newtheorem{definition}[theorem]{Definition}
\newtheorem{example}[theorem]{Example}

\newtheorem{remark}[theorem]{Remark}

\usepackage{ifthen}
\usepackage{intcalc}
\usepackage{tikz}
\usepackage{multirow}
\usetikzlibrary{decorations.pathreplacing,calligraphy}
\usetikzlibrary{arrows.meta, decorations.text, decorations.markings, decorations.pathreplacing, angles, quotes, calligraphy}
\definecolor{bittersweet}{rgb}{1.0, 0.44, 0.37}

\pgfmathsetmacro{\pathyshift}{0.8}
\global\def\rs{{0, 1, 2.1, 3.3, 4.5, 5.5, 6.5}}
\global\def\lsone{{2.15, 2.3, 2.1, 2.4, 2.1, 2.4}}
\global\def\lstwo{{2.1, 2.4, 2.2, 2.5, 2.2, 2.1}}
\global\def\lsthree{{2.3, 2.1, 2.15, 2.35, 2.5, 2.2}}
\global\def\lspathraise{0.7}
\global\def\data{{0.4, 0.7, 0.9, 0.5, 0.6, 0.9, 0.8, 0.5, 0.3, 0.3, 0.8, 1.0, 0.2, 0.4, 0.7, 0.5, 1.0, 1.0, 0.8, 0.4, 0.2, 0.6, 0.8, 0.5, 0.3, 0.5, 0.2, 0.1, 0.5, 0.2, 0.4, 0.6, 0.3, 0.2, 0.1, 0.7, 0.6, 0.2, 0.5, 0.4, 0.4, 0.1}}

\definecolor{colorls1}{HTML}{66C2A5}
\definecolor{colorls2}{HTML}{8DA0CB}
\definecolor{colorls3}{HTML}{FC8D62}

\newcommand\ncdediagram[1]{
    \begin{tikzpicture}
        \pgfmathsetmacro{\logsigversion}{#1}
        
        \ifthenelse{\logsigversion=0}
            {\pgfmathsetmacro{\hreduce}{1.8}}
            {\pgfmathsetmacro{\hreduce}{0}}
        
        \draw[black, thick, ->] (0, 0) -- (7, 0);
        
        \ifthenelse{\logsigversion=0}
            {
                \foreach \i in {0, 6} {
                    \node at (\rs[\i], 0)[circle,fill, inner sep=1.5pt] {};
                }
                \node at (\rs[0], -0.25) {\tiny $t_0$};
                \node at (\rs[6], -0.25) {\tiny $t_m$};
            }
            {
                \foreach \i in {0, 1, 2, 3, 4, 5, 6} {
                    \node at (\rs[\i], 0)[circle,fill, inner sep=1.5pt] {};
                }
                \node at (\rs[0], -0.25) {\tiny $r_0$};
                \node at (\rs[1], -0.25) {\tiny $r_1$};
                \node at (\rs[2], -0.25) {\tiny $r_2$};
                \node at (\rs[3], -0.25) {\tiny $r_3$};
                \node at (\rs[4], -0.25) {\tiny $r_{m-2}$};
                \node at (\rs[5], -0.25) {\tiny $r_{m-1}$};
                \node at (\rs[6], -0.25) {\tiny $r_m$};
            }
        \node at (3.85, -0.25) {\tiny $\cdots$};

        \pgfmathsetmacro{\dataix}{0}
        \pgfmathsetmacro{\midwayprev}{0}
        \pgfmathsetmacro{\xprev}{0}
        \pgfmathsetmacro{\yprev}{0}
        \foreach \i in {0, 1, 2, 3, 4, 5} {
            \pgfmathsetmacro{\frac}{\rs[\i + 1] - \rs[\i]}
            \pgfmathsetmacro{\midway}{\rs[\i] + (\rs[\i + 1] - \rs[\i]) * 0.5}
            
            \foreach \d in {0, 1, 2, 3, 4, 5, 6} {
                \ifthenelse{\i>0 \AND \d=0}{\pgfmathsetmacro{\y}{\data[\dataix - 1] * 0.5 + 0.1}}{\pgfmathsetmacro{\y}{\data[\dataix] * 0.5 + 0.1}}
                \pgfmathsetmacro{\x}{\rs[\i] + \d * \frac * 0.1667}
                \pgfmathsetmacro{\dx}{\frac * 0.1667}
                \ifthenelse{\i=3}
                    {
                        \node at (\x, \y)[circle, draw=black, fill=green!50, inner sep=1.5pt, opacity=0.3] {};
                        \node at (\x, \y + \pathyshift)[circle, draw=black, fill=green!50, inner sep=0.5pt, densely dashed] {};
                        \draw[green!50, thick, cap=round, dashed] (\xprev, \yprev + \pathyshift) -- (\x, \y + \pathyshift);
                    }
                    {
                        \ifthenelse{\logsigversion=0}{\draw (\x, -0.05) -- (\x, 0.05);}{}
                        \node at (\x, \y)[circle, draw=black, fill=green!50, inner sep=1.5pt] {};
                        \node at (\x, \y + \pathyshift)[circle, draw=black, fill=green!50, inner sep=0.5pt] {};
                        \ifthenelse{\i=0 \AND \d=0}{}{\draw[green!50, thick, cap=round] (\xprev, \yprev + \pathyshift) -- (\x, \y + \pathyshift);}
                        \ifthenelse{\logsigversion=0}
                            {
                                \ifthenelse{\d=6 \AND \i=5}
                                    {}
                                    {
                                        \node at (\x, 3.8 - \hreduce)[circle, draw=black, fill=orange!50, inner sep=1.5pt] {};
                                    }
                            }
                            {}
                    }
                \pgfmathsetmacro{\dataixnew}{\dataix + 1}
                \global\let\dataix=\dataixnew
                \global\let\yprev=\y
                \global\let\xprev=\x
                
            }
            
            \ifthenelse{\logsigversion=1}{
                \ifthenelse{\i=3}
                    {
                        \node at (\rs[\i], 3.8)[circle, draw=black, fill=orange!50, inner sep=1.5pt] {};
                        \draw[->, dashed] (\rs[\i] + 0.08, 3.8) -- (\rs[\i+1] - 0.08, 3.8);
                    }
                    {
                        \draw[decoration={calligraphic brace, amplitude=3pt}, decorate, line width=1pt] (\rs[\i] + 0.05, 1.65) node {} -- (\rs[\i+1] - 0.05, 1.65);
                        \draw[->, dashed] (\midway, 1.8) -- (\midway, 2.0);
                        \node at (\midway, \lsone[\i])[circle, draw=black, fill=colorls1, inner sep=1pt] {};
                        \node at (\midway, \lstwo[\i])[circle, draw=black, fill=colorls2, inner sep=1pt] {};
                        \node at (\midway, \lsthree[\i])[circle, draw=black, fill=colorls3, inner sep=1pt] {};
                        \node at (\rs[\i], 3.8)[circle, draw=black, fill=orange!50, inner sep=1.5pt] {};
                        \draw[->] (\rs[\i] + 0.08, 3.8) -- (\rs[\i+1] - 0.08, 3.8);
                    }
                \ifthenelse{\i=3}
                    {
                        
                    }
                    {
                        \node at (\midway, \lsone[\i] + \lspathraise)[circle, draw=black, fill=colorls1, inner sep=.5pt] {};
                        \node at (\midway, \lstwo[\i] + \lspathraise)[circle, draw=black, fill=colorls2, inner sep=.5pt] {};
                        \node at (\midway, \lsthree[\i] + \lspathraise)[circle, draw=black, fill=colorls3, inner sep=.5pt] {};
                    }
                \ifthenelse{\i>0}
                    {   
                        \ifthenelse{\i=3 \OR \i=4}
                        {  
                            \ifthenelse{\i=3}
                                {
                                    \pgfmathsetmacro{\midmidway}{\midwayprev + (\midway - \midwayprev) * 0.5}
                                    \draw[colorls1, thick, cap=round] (\midwayprev, \lsone[\i-1] + \lspathraise) -- (\midmidway, \lsone[\i-1]*0.5 + \lsone[\i]*0.5 + \lspathraise);
                                    \draw[colorls1, thick, cap=round, densely dashed] (\midmidway, \lsone[\i-1]*0.5 + \lsone[\i]*0.5 + \lspathraise) -- (\midway, \lsone[\i] + \lspathraise);
                                    \draw[colorls2, thick, cap=round] (\midwayprev, \lstwo[\i-1] + \lspathraise) -- (\midmidway, \lstwo[\i-1]*0.5 + \lstwo[\i]*0.5 + \lspathraise);
                                    \draw[colorls2, thick, cap=round, densely dashed] (\midmidway, \lstwo[\i-1]*0.5 + \lstwo[\i]*0.5 + \lspathraise) -- (\midway, \lstwo[\i] + \lspathraise);
                                    \draw[colorls3, thick, cap=round] (\midwayprev, \lsthree[\i-1] + \lspathraise) -- (\midmidway, \lsthree[\i-1]*0.5 + \lsthree[\i]*0.5 + \lspathraise);
                                    \draw[colorls3, thick, cap=round, densely dashed] (\midmidway, \lsthree[\i-1]*0.5 + \lsthree[\i]*0.5 + \lspathraise) -- (\midway, \lsthree[\i] + \lspathraise);
                                }
                                {
                                    \pgfmathsetmacro{\midmidway}{\midwayprev + (\midway - \midwayprev) * 0.5}
                                    \draw[colorls1, thick, cap=round, densely dashed] (\midwayprev, \lsone[\i-1] + \lspathraise) -- (\midmidway, \lsone[\i-1]*0.5 + \lsone[\i]*0.5 + \lspathraise);
                                    \draw[colorls1, thick, cap=round] (\midmidway, \lsone[\i-1]*0.5 + \lsone[\i]*0.5 + \lspathraise) -- (\midway, \lsone[\i] + \lspathraise);
                                    \draw[colorls2, thick, cap=round, densely dashed] (\midwayprev, \lstwo[\i-1] + \lspathraise) -- (\midmidway, \lstwo[\i-1]*0.5 + \lstwo[\i]*0.5 + \lspathraise);
                                    \draw[colorls2, thick, cap=round] (\midmidway, \lstwo[\i-1]*0.5 + \lstwo[\i]*0.5 + \lspathraise) -- (\midway, \lstwo[\i] + \lspathraise);
                                    \draw[colorls3, thick, cap=round, densely dashed] (\midwayprev, \lsthree[\i-1] + \lspathraise) -- (\midmidway, \lsthree[\i-1]*0.5 + \lsthree[\i]*0.5 + \lspathraise);
                                    \draw[colorls3, thick, cap=round] (\midmidway, \lsthree[\i-1]*0.5 + \lsthree[\i]*0.5 + \lspathraise) -- (\midway, \lsthree[\i] + \lspathraise);
                                }
                        }
                        {
                            \draw[colorls1, thick, cap=round] (\midwayprev, \lsone[\i-1] + \lspathraise) -- (\midway, \lsone[\i] + \lspathraise);
                            \draw[colorls2, thick, cap=round] (\midwayprev, \lstwo[\i-1] + \lspathraise) -- (\midway, \lstwo[\i] + \lspathraise);
                            \draw[colorls3, thick, cap=round] (\midwayprev, \lsthree[\i-1] + \lspathraise) -- (\midway, \lsthree[\i] + \lspathraise);
                        }
                    }
                    {
                        \draw[colorls1, thick, cap=round] (\midwayprev, \lsone[\i] + \lspathraise) -- (\midway, \lsone[\i] + \lspathraise);
                        \draw[colorls2, thick, cap=round] (\midwayprev, \lstwo[\i] + \lspathraise) -- (\midway, \lstwo[\i] + \lspathraise);
                        \draw[colorls3, thick, cap=round] (\midwayprev, \lsthree[\i] + \lspathraise) -- (\midway, \lsthree[\i] + \lspathraise);
                    }
                \ifthenelse{\i=5} 
                    {
                        \draw[colorls1, thick, cap=round] (\midway, \lsone[\i] + \lspathraise) -- (\rs[\i+1], \lsone[\i] + \lspathraise);
                        \draw[colorls2, thick, cap=round] (\midway, \lstwo[\i] + \lspathraise) -- (\rs[\i+1], \lstwo[\i] + \lspathraise);
                        \draw[colorls3, thick, cap=round] (\midway, \lsthree[\i] + \lspathraise) -- (\rs[\i+1], \lsthree[\i] + \lspathraise);
                    }
                    {}
                }
                {
                }
                
            \global\let\midwayprev=\midway
        }
        
        \ifthenelse{\logsigversion=1}{}{\draw[->, dashed] (\rs[3] + 0.08, 3.8 - \hreduce) -- (\rs[4] - 0.08, 3.8 - \hreduce);}
        
        \draw[blue!50, thick, cap=round] (\rs[0], 4.8 - \hreduce) .. controls (0.5, 3.5 - \hreduce) and (2, 5.3 - \hreduce) .. (\rs[3], 4.8 - \hreduce);
        \draw[blue!50, thick, cap=round, dashed] (\rs[3], 4.8 - \hreduce) .. controls (\rs[3] + 0.2, 4.7 - \hreduce) and (\rs[4] - 0.2, 4.3 - \hreduce) .. (\rs[4], 4.3 - \hreduce);
        \draw[blue!50, thick, cap=round] (\rs[4], 4.3 - \hreduce) .. controls (\rs[4] + 0.3, 4.2 - \hreduce) and (\rs[6] - 0.4, 4.4 - \hreduce) .. (\rs[6], 4.7 - \hreduce);
    
        \ifthenelse{\logsigversion=1}
            {\node at (4, 1.7) {$\cdots$};}
            {}
        
        \node at (8, 0) {\footnotesize Time};
        \node (data) at (8, 0.4) {\footnotesize Data $\mathbf{x}$};
        \node (path) at (8, 1.15) {\footnotesize Path $X$};
        \ifthenelse{\logsigversion=1}{
            \node (logsig) at (8, 2.1) {\footnotesize $\mathrm{LogSig}_{r_i, r_{i+1}}(X)$}; 
            \node (logsig_path) at (8, 2.95) {\footnotesize Log-signature path}; 
        }{}
        \node (hidden) at (8, 4.7 - \hreduce) {\footnotesize Hidden state $Z_t$}; 
    
        \draw[->] (data) -- (path);
        \ifthenelse{\logsigversion=1}
            {
                \draw[->] (path) -- (logsig);
                \draw[->] (logsig) -- (logsig_path);
                \draw [decorate, decoration={calligraphic brace,amplitude=1.5pt, mirror}] (6.5,3.6) -- node[midway, right, xshift=-5.5pt] {\scriptsize \begin{tabular}{l}Integration\\steps\end{tabular}} (6.5,4 - \hreduce);
                \draw[->] (logsig_path) -- (hidden);
            }
            {
                \draw [decorate, decoration={calligraphic brace,amplitude=1.5pt, mirror}] (6.5, 3.6 - \hreduce) -- node[midway, right, xshift=-5.5pt] {\scriptsize \begin{tabular}{l}Integration\\steps\end{tabular}} (6.5, 4 - \hreduce);
                \draw[->] (path) -- (hidden);
            }

    \end{tikzpicture}
}
\def\arrnocomma {
    (0.0, 0.5903005787234183)
    (0.15853658536585366, 0.31548220294099943)
    (0.3170731707317073, 0.5723798869840881)
    (0.47560975609756095, 0.7592559931584444)
    (0.6341463414634146, 0.46000436695787333)
    (0.7926829268292683, 0.8268714937427233)
    (0.9512195121951219, 0.8165360374858748)
    (1.1097560975609757, 0.40954054026875836)
    (1.2682926829268293, 0.7450130964733412)
    (1.4268292682926829, 0.3631221102604148)
    (1.5853658536585367, 0.7482435510764813)
    (1.7439024390243902, 0.917384718293228)
    (1.9024390243902438, 0.5573926930403175)
    (2.0609756097560976, 0.7663748137356283)
    (2.2195121951219514, 0.6694620736634)
    (2.3780487804878048, 1.1176346402925919)
    (2.5365853658536586, 0.949415028614423)
    (2.6951219512195124, 0.3177597410442012)
    (2.8536585365853657, 0.6693951873464801)
    (3.0121951219512195, 0.9294900828561168)
    (3.1707317073170733, 0.9328267954184619)
    (3.3292682926829267, 0.663301293119932)
    (3.4878048780487805, 1.256064569672367)
    (3.6463414634146343, 0.472180216368223)
    (3.8048780487804876, 1.2501090850305931)
    (3.9634146341463414, 0.6636785039153019)
    (4.121951219512195, 0.650984855539094)
    (4.280487804878049, 0.7248161591552023)
    (4.439024390243903, 0.5481046159182468)
    (4.597560975609756, 1.0235344968993427)
    (4.7560975609756095, 1.0369378751495655)
    (4.914634146341464, 1.1262114265391723)
    (5.073170731707317, 0.3962487903451824)
    (5.2317073170731705, 0.9389413773558675)
    (5.390243902439025, 1.2096540128425803)
    (5.548780487804878, 1.079908936992647)
    (5.7073170731707314, 1.2109400195109703)
    (5.865853658536586, 0.37773305265847396)
    (6.024390243902439, 0.5808720577974682)
    (6.182926829268292, 0.5390596533228442)
    (6.341463414634147, 0.6221988967136689)
    (6.5, 0.7650561612631395)
}

\newcommand{\arrcomma}{
    (0.0, 0.5903005787234183), (0.15853658536585366, 0.31548220294099943), (0.3170731707317073, 0.5723798869840881), (0.47560975609756095, 0.7592559931584444), (0.6341463414634146, 0.46000436695787333), (0.7926829268292683, 0.8268714937427233), (0.9512195121951219, 0.8165360374858748), (1.1097560975609757, 0.40954054026875836), (1.2682926829268293, 0.7450130964733412), (1.4268292682926829, 0.3631221102604148), (1.5853658536585367, 0.7482435510764813), (1.7439024390243902, 0.917384718293228), (1.9024390243902438, 0.5573926930403175), (2.0609756097560976, 0.7663748137356283), (2.2195121951219514, 0.6694620736634), (2.3780487804878048, 1.1176346402925919), (2.5365853658536586, 0.949415028614423), (2.6951219512195124, 0.3177597410442012), (2.8536585365853657, 0.6693951873464801), (3.0121951219512195, 0.9294900828561168), (3.1707317073170733, 0.9328267954184619), (3.3292682926829267, 0.663301293119932), (3.4878048780487805, 1.256064569672367), (3.6463414634146343, 0.472180216368223), (3.8048780487804876, 1.2501090850305931), (3.9634146341463414, 0.6636785039153019), (4.121951219512195, 0.650984855539094), (4.280487804878049, 0.7248161591552023), (4.439024390243903, 0.5481046159182468), (4.597560975609756, 1.0235344968993427), (4.7560975609756095, 1.0369378751495655), (4.914634146341464, 1.1262114265391723), (5.073170731707317, 0.3962487903451824), (5.2317073170731705, 0.9389413773558675), (5.390243902439025, 1.2096540128425803), (5.548780487804878, 1.079908936992647), (5.7073170731707314, 1.2109400195109703), (5.865853658536586, 0.37773305265847396), (6.024390243902439, 0.5808720577974682), (6.182926829268292, 0.5390596533228442), (6.341463414634147, 0.6221988967136689), (6.5, 0.7650561612631395)
}

\global\def\xarrm{{0.0, 0.15853658536585366, 0.3170731707317073, 0.47560975609756095, 0.6341463414634146, 0.7926829268292683, 0.9512195121951219, 1.1097560975609757, 1.2682926829268293, 1.4268292682926829, 1.5853658536585367, 1.7439024390243902, 1.9024390243902438, 2.0609756097560976, 2.2195121951219514, 2.3780487804878048, 2.5365853658536586, 2.6951219512195124, 2.8536585365853657, 3.0121951219512195, 3.1707317073170733, 3.3292682926829267, 3.4878048780487805, 3.6463414634146343, 3.8048780487804876, 3.9634146341463414, 4.121951219512195, 4.280487804878049, 4.439024390243903, 4.597560975609756, 4.7560975609756095, 4.914634146341464, 5.073170731707317, 5.2317073170731705, 5.390243902439025, 5.548780487804878, 5.7073170731707314, 5.865853658536586, 6.024390243902439, 6.182926829268292, 6.341463414634147, 6.5}}

\global\def\yarrm{{0.5903005787234183, 0.31548220294099943, 0.5723798869840881, 0.7592559931584444, 0.46000436695787333, 0.8268714937427233, 0.8165360374858748, 0.40954054026875836, 0.7450130964733412, 0.3631221102604148, 0.7482435510764813, 0.917384718293228, 0.5573926930403175, 0.7663748137356283, 0.6694620736634, 1.1176346402925919, 0.949415028614423, 0.3177597410442012, 0.6693951873464801, 0.9294900828561168, 0.9328267954184619, 0.663301293119932, 1.256064569672367, 0.472180216368223, 1.2501090850305931, 0.6636785039153019, 0.650984855539094, 0.7248161591552023, 0.5481046159182468, 1.0235344968993427, 1.0369378751495655, 1.1262114265391723, 0.3962487903451824, 0.9389413773558675, 1.2096540128425803, 1.079908936992647, 1.2109400195109703, 0.37773305265847396, 0.5808720577974682, 0.5390596533228442, 0.6221988967136689, 0.7650561612631395}}

\global\def\xarrlen{0, 1, 2, 3, 4, 5, 6, 7, 8, 9, 10, 11, 12, 13, 14, 15, 16, 17, 18, 19, 20, 21, 22, 23, 24, 25, 26, 27, 28, 29, 30, 31, 32, 33, 34, 35, 36, 37, 38, 39, 40, 41}

\newcommand\mypath[1]{
    \draw [green!50, line width=0.5mm] plot [#1] coordinates {
        \arrnocomma
    };
}

\newcommand\ncdenrdediagram[1]{
    \begin{tikzpicture}
        \pgfmathsetmacro{\logsigversion}{#1}
    
        \draw[black, thick, ->] (0, 0) -- (7, 0);
        
        \pgfmathsetmacro{\summaryh}{1.7}
        \pgfmathsetmacro{\summaryeps}{0.01}
        \pgfmathsetmacro{\summaryd}{0.5}
        \pgfmathsetmacro{\cdeh}{2.5}
        \pgfmathsetmacro{\cded}{1}
        
        \draw (\rs[0], -0.05) -- (\rs[0], 0.05);
        \draw (\rs[6], -0.05) -- (\rs[6], 0.05);
        \node at (\rs[0], -0.25) {\small $t_0$};
        \node at (\rs[6], -0.25) {\small $T$};

        \ifthenelse{\logsigversion=1}
            {
                \foreach \d in {1, 2, 3, 4, 5, 6} {
                    \draw [fill=red!10, draw=red!50, rounded corners, dashed] (\rs[\d - 1] + \summaryeps, \summaryh) rectangle (\rs[\d] - \summaryeps, \summaryh + \summaryd) node[pos=.5] {};
                    
                    \draw [black!50, ->] (0.5 * \rs[\d - 1] + 0.5 * \rs[\d], \summaryh + \summaryd) -- (0.5 * \rs[\d - 1] + 0.5 * \rs[\d], \cdeh);
                    
                    \draw [black!50, ->] (0.5 * \rs[\d - 1] + 0.5 * \rs[\d], \cdeh + \cded) -- (0.5 * \rs[\d - 1] + 0.5 * \rs[\d], \cded + \cdeh + 0.5);
                }
            }
            {
            }
            
        \pgfmathsetmacro{\cdeeps}{0.}
        \ifthenelse{\logsigversion=0}
            {
                \draw [fill=orange!10, draw=orange!50, rounded corners, dashed] (\rs[0], \cdeh - \cdeeps) rectangle (\rs[6], \cdeh + \cded - \cdeeps) node[pos=.5] {Neural CDE};
            }
            {
                \draw [fill=orange!10, draw=orange!50, rounded corners, dashed] (\rs[0], \cdeh) rectangle (\rs[6], \cdeh + \cded) node[pos=.5] {Neural RDE};
            }
        
        \ifthenelse{\logsigversion=1}{
                \mypath{}
            }
            {
                \mypath{smooth}
            }
        \foreach \coord [count=\i] in \arrcomma {
            \node at \coord [circle, draw=black, fill=green!50, inner sep=0.5pt] {};
        }
        \foreach \i in \xarrlen {
            \ifthenelse{\logsigversion=1}{
                \draw [black!50, ->] (\xarrm[\i], \yarrm[\i]) -- (\xarrm[\i], \summaryh);
            }
            {
                \draw [black!50, ->] (\xarrm[\i], \yarrm[\i]) -- (\xarrm[\i], \cdeh);
                \draw [black!50, ->] (\xarrm[\i], \cdeh + \cded) -- (\xarrm[\i], \cdeh + \cded + 0.5);
            }
        }
        
        \pgfmathsetmacro{\hiddenh}{1.5}
        \draw[blue!50, line width=0.5mm, cap=round] (\rs[0], 3 + \hiddenh) .. controls (1.5, 1.6 + \hiddenh) and (\rs[4], 4 + \hiddenh) .. (\rs[6], 2.5 + \hiddenh);
        
        \node at (8, 0) {\footnotesize Time};
        \ifthenelse{\logsigversion=0}{
                \node (cde) at (8, 3) {CDE}; 
                \node (path) at (8, 0.8) {\footnotesize \begin{tabular}{c}Path, $X_t$\\(smoothed)
                \end{tabular}};
            }
            {
                \node (path) at (8, 0.8) {\footnotesize \begin{tabular}{c}Path, $X_t$\\ \end{tabular}};
            }
        \ifthenelse{\logsigversion=1}{
            \node (summarisations) at (8, 1.95) {\footnotesize Summaries}; 
            \node (rde) at (8, 3) {RDE}; 
        }{}
        \node (hidden) at (8, 4.2) {\footnotesize Response};
    
        \ifthenelse{\logsigversion=1}
            {
                \draw[->] (path) -- (summarisations);
                \draw[->] (summarisations) -- (rde);
                \draw[->] (rde) -- (hidden);
            }
            {
                \draw[->] (path) -- (cde);
                \draw[->] (cde) -- (hidden);
            }

    \end{tikzpicture}
}

\tikzset{middlearrow/.style={
        decoration={markings,
            mark= at position 0.5 with {\arrow{#1}} ,
        },
        postaction={decorate}
    }
}
\tikzset{input/.style={black, draw=green!50, fill=green!50, rectangle, minimum height=0.8cm}}
\tikzset{hidden/.style={black, draw=blue!50, fill=blue!50, rectangle, minimum height=0.8cm}}
\tikzset{hidden_square/.style={black, draw=blue!50, fill=blue!50, rectangle, minimum height=3.5cm}}
\tikzset{logsig/.style={black, draw=red!50, fill=red!50, rectangle, minimum height=0.8cm}}

\tikzset{>=latex}
\usepackage{pgfplots} 
\pgfplotsset{compat=1.10}
\usepgfplotslibrary{fillbetween}
\usetikzlibrary{patterns}

\newcommand{\logsig}{\mathrm{LogSig}}
\newcommand{\dby}{\mathrm{d}}
\newcommand{\reals}{\mathbb{R}}
\newcommand{\naturals}{\mathbb{N}}
\newcommand{\restr}[2]{{\left.\kern-\nulldelimiterspace #1 \right|_{#2}}}

\usepackage{hyperref}

\usepackage[accepted]{style/icml2021}

\icmltitlerunning{Neural Rough Differential Equations for Long Time Series}

\usepackage{style/widetext}

\begin{document}

\twocolumn[
\icmltitle{Neural Rough Differential Equations for Long Time Series}




\begin{icmlauthorlist}
\icmlauthor{James Morrill}{to,goo}
\icmlauthor{Cristopher Salvi}{to,goo}
\icmlauthor{Patrick Kidger}{to,goo}
\icmlauthor{James Foster}{to,goo}
\icmlauthor{Terry Lyons}{to,goo}

\end{icmlauthorlist}

\icmlaffiliation{to}{Mathematical Institute, University of Oxford, UK}
\icmlaffiliation{goo}{The Alan Turing Institute, British Library, UK}

\icmlcorrespondingauthor{James Morrill}{morrill@maths.ox.ac.uk}

\icmlkeywords{neural differential equations, neural odes, neural cdes, neural rdes, neural ordinary differential equations, neural controlled differential equations, neural rough differential equations, machine learning, deep learning, signatures, rough analysis, rough path theory, time Series, long time series}

\vskip 0.3in
]



\printAffiliationsAndNotice{}  

\begin{abstract}
Neural controlled differential equations (CDEs) are the continuous-time analogue of recurrent neural networks, as Neural ODEs are to residual networks, and offer a memory-efficient continuous-time way to model functions of potentially irregular time series. Existing methods for computing the forward pass of a Neural CDE involve embedding the incoming time series into path space, often via interpolation, and using evaluations of this path to drive the hidden state. Here, we use rough path theory to extend this formulation. Instead of directly embedding into path space, we instead represent the input signal over small time intervals through its \textit{log-signature}, which are statistics describing how the signal drives a CDE. This is the approach for solving \textit{rough differential equations} (RDEs), and correspondingly we describe our main contribution as the introduction of Neural RDEs. This extension has a purpose: by generalising the Neural CDE approach to a broader class of driving signals, we demonstrate particular advantages for tackling long time series. In this regime, we demonstrate efficacy on problems of length up to 17k observations and observe significant training speed-ups, improvements in model performance, and reduced memory requirements compared to existing approaches.

\end{abstract}

\section{Introduction}
Neural controlled differential equations (CDEs) \citep{kidger2020neural} are the continuous-time analogue to recurrent neural networks (RNNs) and provide a natural method for modelling temporal dynamics with neural networks.

Neural CDEs are similar to neural ordinary differential equations (ODEs), as popularised by \citet{neural2018ode}. A Neural ODE is determined by its initial condition, without a direct way to modify the trajectory given subsequent observations. In contrast, the vector field of a Neural CDE depends upon the time-varying data, so that the trajectory of the system is driven by a sequence of observations.

\subsection{Controlled Differential Equations}

Let $a, b \in \mathbb{R}$ with $a < b$, and let $v, w \in \mathbb{N}$. Let $\xi \in \mathbb{R}^w$. Let $X \colon [a, b] \to \mathbb{R}^v$ be a continuous function of bounded variation (which is for example implied by it being Lipschitz), and let $f \colon \mathbb{R}^w~\to~\mathbb{R}^{w \times v}$ be continuous.

Then we may define $Z \colon [a, b] \to \mathbb{R}^w$ as the unique solution to the \textit{controlled differential equation}
\begin{equation}
    Z_a = \xi,\quad Z_t = Z_a + \int^{t}_{a} f(Z_s)\,\dby X_s \quad \text{for } t \in (a, b].
    \label{eq:kidger_cde}
\end{equation}
The notation ``$f(Z_s) \dby X_s$'' denotes a matrix-vector product. ``$\dby X_s$'' itself denotes a Riemann--Stieltjes integral: if $X$ is differentiable then 
\begin{multline}
        \int_a^t f(Z_s) \dby X_s = \int_a^t f(Z_s) \dot X_s \dby s,\\ \text{with} \; \dot X_s = \frac{\dby X_r}{\dby r}(s).
\end{multline}

If in equation (\ref{eq:kidger_cde}), $\dby X_s$ was replaced with $\dby s$, then the equation would just be an ODE. Using $\dby X_s$ causes the solution to depend continuously on the evolution of $X$. We say that the solution is ``driven by the control $X$".

Next, we recall the definition of a Neural CDE as introduced in \citet{kidger2020neural}.

\begin{figure*}
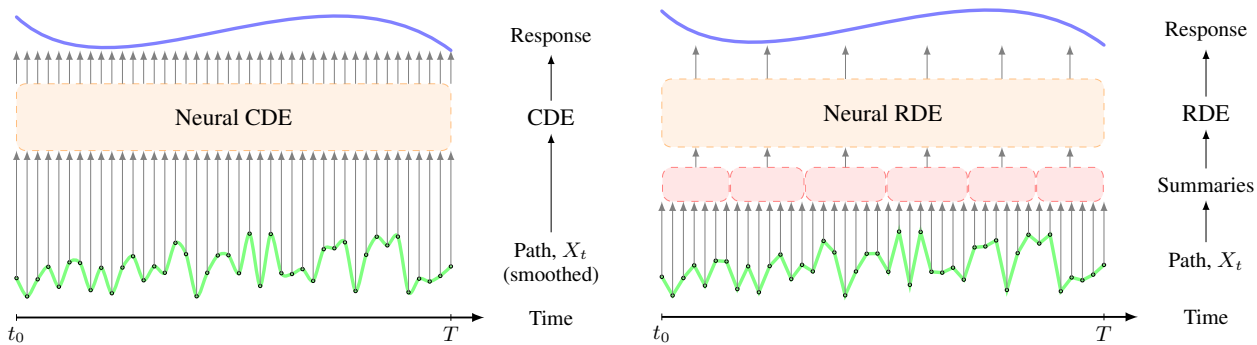

    \vspace{-1em}
    \begin{subfigure}[t]{.5\linewidth}
        \centering
            \resizebox{\linewidth}{!}{
                \ncdenrdediagram{0}
            }
    \end{subfigure}%
    \begin{subfigure}[t]{.5\linewidth}
        \centering
        \resizebox{\linewidth}{!}{
            \ncdenrdediagram{1}
        }
    \end{subfigure}
    \vspace{-1.5em}
    \caption{Here we give a high level comparison of the CDE and the RDE formulations. \textbf{Left:} The original CDE formulation where the data is smoothly interpolated and pointwise derivative information is used to drive the CDE. \textbf{Right:} The corresponding rough approach. Local interval summarisations of the data are computed and used to drive the response over the interval.}
    \label{fig:ncde_nrde_plot}
\end{figure*}

\subsection{Neural Controlled Differential Equations}

Consider a time series $\mathbf{x}$ as a collection of points $x_i \in \mathbb{R}^{v-1}$ with corresponding time-stamps $t_i \in \mathbb{R}$ such that $\mathbf{x} = ((t_0, x_0), (t_1, x_1), ..., (t_n, x_n))$, and $t_0 < ... < t_n$.

Let $X \colon [t_0, t_n] \to \reals^{v}$ be some interpolation of the data such that $X_{t_i} = (t_i, x_i)$. In \citet{kidger2020neural} the authors use natural cubic splines to ensure differentiability of the control $X$, so as to treat the term ``$\dby X_s$'' in equation (\ref{eq:kidger_cde}) as ``$\dot X_s \dby s$''. 


Let $\xi_\theta \colon \reals^{v} \to \reals^w$ and $f_\theta \colon \reals^w \to \reals^{w \times v}$ be two neural networks and let $\ell_\theta \colon \reals^w \to \reals^{q}$ be a linear map, for some output dimension $q \in \naturals$. Here $\theta$ is used to denote dependence on learnable parameters.

We define $Z$ as the hidden state and $Y$ as the output of a Neural CDE driven by $X$ if
\begin{multline}
    Z_{t_0} = \xi_\theta(t_0, x_0), \text{ with } Z_t = Z_{t_0} + \int^{t}_{t_0} f_\theta(Z_s)\,\dby X_s, \\
    \text{and } Y_t = \ell_\theta(Z_t) \text{ for } t \in (t_0, t_n]
    \label{eq:kidger_ncde}
\end{multline}
That is -- just like an RNN -- we have an evolving hidden state $Z$, which is fed into a linear map to produce an output $Y$. This formulation is a universal approximator \citep[Appendix B]{kidger2020neural}. The output may be either the time-evolving $Y_t$ or just the final $Y_{t_n}$. This is then fed into a loss function ($L^2$, cross entropy, \ldots) and trained via stochastic gradient descent in the usual way.

To compute the integral of equation (\ref{eq:kidger_ncde}) in \citet{kidger2020neural}, $X$ is assumed differentiable and the CDE is simply rewritten as an ODE of the form
\begin{equation}
     Z_t = Z_{t_0} + \int^{t}_{t_0} g_{\theta, X}(Z_s, s)\,\dby s,
     \label{eq:kidger_cde_evaluation}
\end{equation}
where 
\begin{equation}
    g_{\theta, X}(Z, s) = f_\theta(Z) \dot X_s. 
    \label{eq:kidger_g_X}
\end{equation}
This simple observation allows for incorporating the time-varying data $X$ driving the CDE into the vector field $g_{\theta,X}$ of the equivalent ODE (\ref{eq:kidger_cde_evaluation}). In doing so existing tools for Neural ODEs can be used to carry out the forward pass and backpropagate via adjoint methods.





\colorlet{lightgray}{gray!10}
\definecolor{colorp1}{HTML}{4878d0}
\definecolor{colorp2}{HTML}{ee854a}
\definecolor{colorp3}{HTML}{6acc64}
\definecolor{colord1}{HTML}{003f5c}
\definecolor{colord2}{HTML}{bc5090}
\definecolor{colord3}{HTML}{ff6361}
\definecolor{colorls1}{HTML}{66C2A5}
\definecolor{colorls2}{HTML}{8DA0CB}
\definecolor{colorls3}{HTML}{FC8D62}

\global\def\xs{{0.        , 0.13157895, 0.26315789, 0.39473684, 0.52631579,
       0.65789474, 0.78947368, 0.92105263, 1.05263158, 1.18421053,
       1.31578947, 1.44736842, 1.57894737, 1.71052632, 1.84210526,
       1.97368421, 2.10526316, 2.23684211, 2.36842105, 2.5}}
\global\def\ys{{0.        , 0.48565753, 0.84647908, 1.09613282, 1.24828692,
       1.31660956, 1.31476892, 1.25643315, 1.15527045, 1.02494897,
       0.8791369 , 0.73150241, 0.59571366, 0.48543884, 0.41434611,
       0.39610366, 0.44437965, 0.57284225, 0.79515964, 1.125}}
       
\global\def\plotrange{0, 1, 2, 4, 5, 8, 11, 12, 13, 15, 18, 19}

\begin{figure*}[t]
    \centering
    \resizebox{0.86\linewidth}{!}{
        \begin{tikzpicture}
            \draw[draw=white, fill=lightgray] (0, 0) grid (5, 3) rectangle (0, 0);

                \foreach \i in \plotrange {
                    \node at (\xs[\i] * 2, 0.5 + \ys[\i] * 1.5)[circle, draw=black, fill=green!50, inner sep=1.5pt] {};
                }
            \node at (2.5, 3) [above=2mm] {Data, $\mathbf{x}$};
            \draw[black, ->] (0, 0) -- (0, 3.03);
            \draw[black, ->] (0, 0) -- (5.03, 0);
            \node at (5.3, 0) {\small $X^1$};
            \node at (0, 3.3) {\small $X^2$};
            
            \draw[->, line width=0.2mm] (5.2, 1.5) -- (5.9, 1.5);
            
            \draw[draw=white, fill=lightgray] (6, 0) grid (11, 3) rectangle (6, 0);
            \pgfmathsetmacro{\xprev}{6 + 0.05}
            \pgfmathsetmacro{\yprev}{0+0.5}
            \draw[colorls1, line width=0.5mm, dashed, ->] (\xprev + 0.05, \yprev) -- (\xs[19] * 2 + 6 - 0.1, \yprev);
            \draw[colorls2, line width=0.5mm, dashed, ->] (\xs[19] * 2 + 6, \yprev) -- (\xs[19] * 2 + 6, 0.5 + \ys[19] * 1.5 - 0.1);
            \node at (6 + 2.5, \yprev) [above] {\footnotesize $\Delta X^1$};
            \node at (\xs[19] * 2 + 6, 1)  [right] {\footnotesize $\Delta X^2$};
            \draw[green!50, xshift=6cm, name path=one] plot coordinates {
                (0, 0.5 + 1.5 * 0) (2 * 0.13157894736842105, 0.5 + 1.5 * 0.48565753025222336) (2 * 0.2631578947368421, 0.5 + 1.5 * 0.846479078582884) (0.5263157894736842 * 2, 0.5 + 1.5 * 1.2482869222918795) (0.6578947368421052 * 2, 0.5 + 1.5 * 1.3166095640763957) (1.0526315789473684 * 2, 0.5 + 1.5 * 1.1552704475871118) (1.4473684210526314 * 2, 0.5 + 1.5 * 0.731502405598484) (1.5789473684210527 * 2, 0.5 + 1.5 * 0.5957136608835107) (1.7105263157894737 * 2, 0.5 + 1.5 * 0.4854388394809739) (1.9736842105263157 * 2, 0.5 + 1.5 * 0.39610365942557224) (2.3684210526315788 * 2, 0.5 + 1.5 * 0.7951596442630118) (2.5 * 2, 0.5 + 1.5 * 1.125)
            };
            \draw[green!50, xshift=6cm, name path=two] plot coordinates {
                (0, 0.5) (2.5 * 2, 0.5 + 1.5 * 1.125)
            };
            \tikzfillbetween[
                of=one and two, split
            ] {pattern=north west lines, colorls3!50};
            \node at (7.1, 1.7) [right=0.1mm] {\footnotesize $A_{-}$};
            \node at (9.6, 1.5) [right=0.1mm] {\footnotesize $A_{+}$};
            
            \foreach \i in \plotrange {
                \pgfmathsetmacro{\x}{\xs[\i] * 2 + 6}
                \pgfmathsetmacro{\y}{0.5 + \ys[\i] * 1.5}
                \node at (\x, \y)[circle, draw=black, fill=green!50, inner sep=1.5pt] {};
                \draw[green!50, thick, cap=round] (\xprev, \yprev) -- (\x, \y);
                \global\let\yprev=\y
                \global\let\xprev=\x
            }
            \node at (6 + 2.5, 3) [above=2mm] {Path, $X$};
            \draw[black, ->] (6, 0) -- (6, 3.03);
            \draw[black, ->] (6, 0) -- (6 + 5.03, 0);
            \node at (6 + 5.3, 0) {\small $X^1$};
            \node at (6, 3.3) {\small $X^2$};
            
            \draw[->, line width=0.2mm] (11.2, 1.5) -- (11.9, 1.5);
        
            \pgfmathsetmacro{\shift}{0}
            \pgfmathsetmacro{\xend}{14}
            \draw[fill=lightgray, lightgray] (12 + \shift, 0) rectangle (\xend, 3);
            \filldraw[colorls1] (12.3 + \shift, 2.75) circle (4pt);
            \filldraw[colorls2] (12.3 + \shift, 2.35) circle (4pt);
            \filldraw[colorls3] (12.3 + \shift, 1.95) circle (4pt);
            \node at (12.4 + \shift, 2.8) [right] {\tiny $= \Delta X^1$};
            \node at (12.4 + \shift, 2.4) [right] {\tiny $= \Delta X^2$};
            \node at (12.4 + \shift, 1.93) [right] {\tiny $= A_{+} - A_{-}$};
            \node[black, rotate=90] at (12.3 + \shift, 0.95) {\Large \ldots\ldots};
            \node at (12 + 1 + 0.5*\shift, 3) [above=2mm] {Log-signature};
            \draw [decorate, decoration={calligraphic brace, amplitude=1.5pt, mirror}] (\xend + 0.1, 2.3) -- node[midway, right, xshift=-0.5pt] {\tiny Depth 1} (\xend + 0.1, 2.9);
            \draw [decorate, decoration={calligraphic brace, amplitude=1.5pt, mirror}] (\xend + 0.1, 1.75) -- node[midway, right, xshift=-0.5pt] {\tiny Depth 2} (\xend + 0.1, 2.15);
            \draw [decorate, decoration={calligraphic brace, amplitude=1.5pt, mirror}] (\xend + 0.1, 0.2) -- node[midway, right, xshift=-0.5pt] {\tiny Higher order} (\xend + 0.1, 1.6);
        
        \end{tikzpicture}
    }
    \caption{Geometric intuition for the first two levels of the log-signature for a 2-dimensional path. The depth 1 terms correspond to the change in each of the coordinates over the interval. The depth 2 term corresponds to the \textit{L\'evy area} of the path, this being the signed area between the curve and the chord joining its start and endpoints.}
    \label{fig:geometric_signature}
\end{figure*}
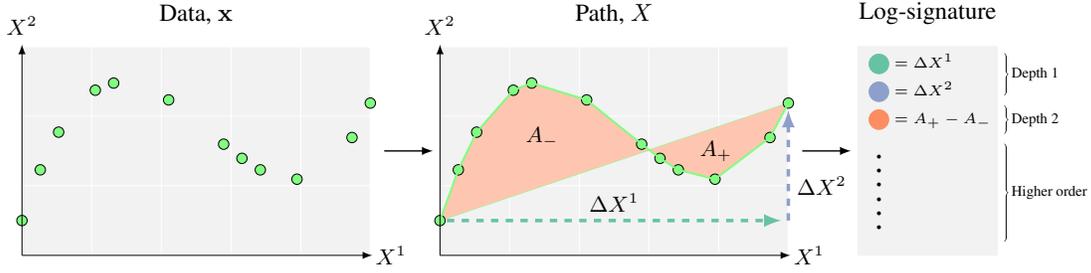

 

\subsection{Contributions}
Neural CDEs, as with RNNs, begin to break down for long time series. Loss/accuracy worsens, and training time becomes prohibitive due to the sheer number of forward operations within each training epoch.

Meanwhile, and at first glance tangentially, it is known in the field of rough path theory \cite{lyons1998differential, lyons2004differential, friz2010multidimensional} that it is possible to numerically solve CDEs not by pointwise evaluations of the control path (as in the existing Neural CDE approach), but by using a specific summarisation -- known as \emph{the log-signature} -- of the control path over short time intervals. See Figure \ref{fig:ncde_nrde_plot}. A CDE treated in this way is termed a \textit{rough differential equation}, and the numerical method is termed \textit{the log-ODE method}.

The central contribution of this paper is to observe that this latter technique actually offers a way to solve the former problem. The log-ODE method offers a way to update the hidden state of a Neural CDE over large intervals -- much larger than would be expected given the sampling rate or length of the data. This dramatically reduces the effective length of the time series. Log-signatures represents a CDE-specific choice of summarisation, which works because closely-spaced samples are often strongly correlated. Additionally, this approach no longer requires differentiability of the control path.

In line with the usual mathematical terminology, we refer to our approach as \textit{neural rough differential equations} (Neural RDEs). Moreover, Neural RDEs are still able to exploit memory-efficient continuous-time adjoint backpropagation. This is of additional benefit as memory pressure becomes increasingly relevant for long time series -- indeed many of our experiments could not have been ran without it.

With Neural RDEs, we demonstrate improvements experimentally on real-world problems of length up to 17 000. We report substantial improvements in model performance (by as much as 17\% on some classification tasks, reflecting the difficulty inherent in long time series), speed (by roughly a factor of 10), and memory usage (by roughly a factor of 100 compared to models not using the adjoint method).

\section{Theory}\label{sec:theory}
We begin with an exposition on the motivating theory. Our description here will focus on the high-level intuitions. For a full technical description we refer to the appendices; see also section 7.1 of \cite{friz2010multidimensional}.

Readers primarily interested in practical applications should feel free to skip to section \ref{sec:method}.

\subsection{Signatures and Log-signatures} \label{subsec:signatures}

The signature transform is a map from paths to a vector of real values, specifying a collection of statistics about the path. It is a central component of the theory of controlled differential equations since these statistics describe how the data interacts with dynamical systems. The log-signature is then formed by representing the same information in a compressed format.

\paragraph{Signature transform} Let $X = (X^1, ..., X^d): [0, T] \rightarrow \mathbb{R}^d$ be continuous and piecewise differentiable.\footnote{For our purposes later it will typically be a linear interpolation of a time series.} Letting\footnote{This is a slightly simplified definition, and the signature is often instead defined using the notation of stochastic calculus; for completeness see Definition \ref{def:signature_appendix}.}
\begin{equation}
    S^{i_1,...i_k}_{a, b}(X) = \underset{\scriptscriptstyle a < t_1 < ... < t_k < b}{\int ... \int} \prod^k_{j=1} \frac{\dby X^{i_j}}{\dby t} (t_j) \dby t_j,
    \label{eq:signature}
\end{equation}
then the depth-N signature transform of $X$ is given by
\begin{multline}
    \mathrm{Sig}^N_{a, b}(X) = \Big(\big\{S^i_{a, b}(X)^{(i)}\big\}_{i = 1}^{d}, \big\{S^{i, j}_{a, b}(X)\big\}_{i, j = 1}^{d}, \\ 
    \ldots, \big\{S^{i_1,\ldots, i_N}_{a,b}(X)\big\}_{i_1, \ldots, i_N = 1}^{d}\Big).
    \label{eq:truncated_path_signature}
\end{multline}

This definition is independent of the choice of $T$ and $t_i$, by change of variables in equation (\ref{eq:signature}).

We see that the signature is a collection of integrals, with each integral defining a real value. It is a graded sequence of statistics that characterise the input time series. In particular, \citep{hambly2010sigunique} show that under mild conditions, $\mathrm{Sig}^\infty(X)$ completely determines $X$ up to translation, provided time is included as a channel in $X$.

\paragraph{Log-signature transform} The signature transform has some redundancy: a little algebra shows that for example $S^{1, 2}_{a, b}(X) + S^{2, 1}_{a, b}(X) = S^1_{a, b}(X) S^2_{a, b}(X)$, so that we already know $S^{2, 1}_{a, b}(X)$ provided we know the other three quantities.

The \emph{log-signature transform} is then essentially obtained by computing the signature transform, and throwing out redundant terms, to obtain some (nonunique) minimal collection. 





Starting from the depth-N signature transform and removing some fixed set of redundancies produces the \emph{depth-N log-signature transform}. We fix some set of redundancies throughout (essentially corresponding to a choice of basis), and denote this $\logsig^N_{a, b}$. This is a map from Lipschitz continuous paths $[a, b] \to \reals^v$ into $\reals^{\beta(v, N)}$, where $\beta(v, N)$ denotes the dimension of the log-signature (see Appendix \ref{eq:logsig-dim}).

\paragraph{Geometric intuition} In figure \ref{fig:geometric_signature} we provide a geometric intuition for the first two levels of the log-signature, which have natural geometric interpretations.

The depth 1 terms correspond to the changes in each channel over the interval; this is $\Delta X_1, \Delta X_2$ in the figure. The depth 2 term corresponds to the signed area in between the chord joining the endpoints and the path itself; this corresponds to $A_{+} - A_{-}$ in the figure. Higher order terms correspond to higher order integrals and iterated areas in higher dimensional spaces, and become a little more difficult to visualise.

\begin{figure*}[t]
    \centering
    \vspace{-0.5em}
    \begin{align*}
    Z_t &\approx Z_a + \int^{t}_{a} \Big(f(Z_a) + D_f(Z_a)(Z_s - Z_a)\Big) \frac{\dby X}{\dby t}(s)\m\dby s\nonumber\\[2pt]
    &= Z_a + \int^{t}_{a} f(Z_a)\frac{\dby X}{\dby t}(s)\,\dby s\nonumber
    + \int^{t}_{a} \bigg(D_f(Z_a)\int^{s}_{a} f(Z_u)\, \frac{\dby X}{\dby t}(u)\,\dby u\bigg)\frac{\dby X}{\dby t}(s)\,\dby s\nonumber\\[2pt]
    &\approx Z_a + f(Z_a)\int^{t}_{a} \frac{\dby X}{\dby t}(s)\,\dby s\nonumber
     + D_f(Z_a)f(Z_a)\int^{t}_{a}\int^{s}_{a} \frac{\dby X}{\dby t}(u)\,\dby u\frac{\dby X}{\dby t}(s)\m\dby s\nonumber\\[2pt]
    &= Z_a + f(Z_a)\big\{S(X)^{(i)}\}_{i = 1}^{d}\nonumber
    + D_f(Z_a)f(Z_a)\big\{S(X)^{(i, j)}\big\}_{i, j = 1}^{d}.
    \end{align*}
    \vspace{-1.5em}
    \caption{Signature (Taylor) expansion of a CDE. The action of the vector field f on the depth-N signature is a matrix-vector product and is fully described, for any N, in \cite{logode2014estimate}.}\label{eq:simple_taylor}
\end{figure*}

\paragraph{(Log-)Signatures and CDEs} In Figure \ref{eq:simple_taylor} we give the equations for how log-signatures arise in the solution of CDEs. Begin by letting $D_f$ denote the Jacobian of a function $f$. Now expand equation (\ref{eq:kidger_cde}) by linearising the vector field $f$ and neglecting higher order terms.

This is simply the Taylor Expansion of the CDE. The Taylor coefficients are precisely these signature terms, thus demonstrating how signatures are intrinsically linked to the solutions of CDEs. Higher order Taylor expansions results in corrections using higher order signature terms.



\subsection{The Log-ODE Method}
Recall for $X \colon [a, b] \to \reals^v$ that $\logsig^N_{a, b}(X) \in \reals^{\beta(v, N)}$. The log-ODE method states that $Z_b \approx \widehat{Z}_b$ where
\begin{equation}
    \widehat{Z}_u = \widehat{Z}_a + \int^{u}_{a} \widehat{f}(\widehat{Z}_s) \frac{\logsig^N_{a, b}(X)}{b-a}\dby s \text{ for } u \in (a, b],
    \label{eq:log-ode}
\end{equation}
and $\widehat{Z}_{a} = Z_{a}$. Here $Z$ is the same as in equation (\ref{eq:kidger_ncde}), and the relationship between $\widehat{f}$ to $f$ is given in Appendix \ref{apx:logode}.

That is, the solution of the CDE may be approximated by the solution to an ODE. This is typically applied locally: pick some points $r_i$ such that $a = r_0 < r_1 < \cdots < r_m = b$, split up the CDE of equation (\ref{eq:kidger_cde}) into an integral over $[r_0, r_1]$, an integral over $[r_1, r_2]$, and so on, and apply the log-ODE method to each interval separately. A CDE treated in this way is, for the purposes of this exposition, termed a \textit{rough differential equation}.

See Appendix \ref{apx:logode} for the precise details and Appendix \ref{apx:logodeconv} for a proof of convergence. For the reader familiar with the Magnus expansion for linear differential equations \citep{magnus2008expansion}, then the log-ODE method is a generalisation.



\begin{figure*}[!hbtp]
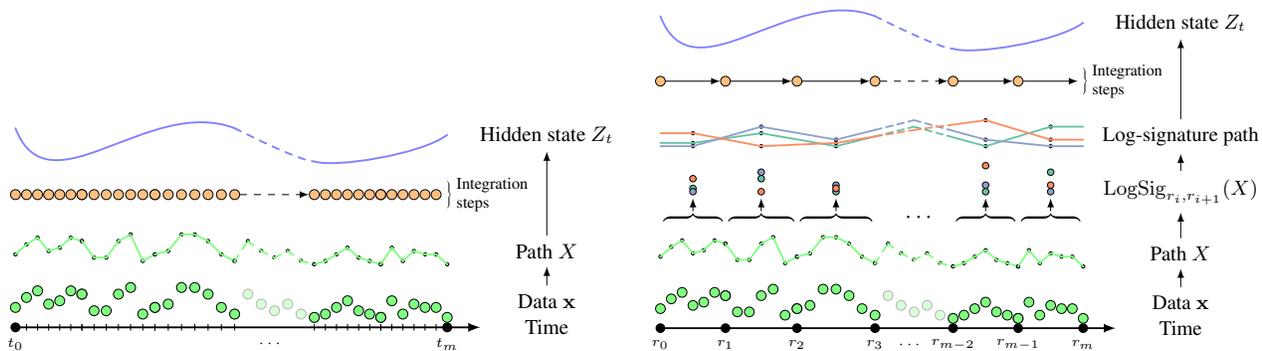

    \begin{subfigure}[t]{.5\linewidth}
        \centering
            \resizebox{\linewidth}{!}{
                \ncdediagram{0}
            }
    \end{subfigure}%
    \begin{subfigure}[t]{.5\linewidth}
        \centering
        \resizebox{\linewidth}{!}{
            \ncdediagram{1}
        }
    \end{subfigure}
    \caption{An overview of the log-ODE method applied to Neural RDEs. \textbf{Left:} A single step (CDE or RDE) model. The path $X$ is quickly varying, meaning a lot of integration steps are needed to resolve it. \textbf{Right:} The Neural RDE utilising the log-ODE method with integration steps larger than the discretisation of the data. The path of log-signatures is more slowly varying (in a higher dimensional space), and needs fewer integration steps to resolve.}
    \label{fig:ncde_plots}
\end{figure*}

\section{Method}\label{sec:method}
We move on to introducing the neural rough differential equation.

Recall that we observe some time series $\mathbf{x} = ((t_0, x_0), (t_1, x_1), ..., (t_n, x_n))$, and have constructed a piecewise linear interpolation $X \colon [t_0, t_n] \to \mathbb{R}^{v}$ such that $X_{t_i} = (t_i, x_i)$.

We now pick points $r_i$ such that $t_0 = r_0 < r_1 < \cdots < r_m = t_n$. In principle these can be variably spaced but in practice we will typically space them equally far apart. The total number of points $m$ should be much smaller than $n$. The choice and spacing of $r_i$ will be a hyperparameter.

We also pick a depth hyperparameter $N \geq 1$. In section \ref{sec:theory} we introduced the depth-$N$ log-signature transform. For $X \colon [t_0, t_n] \to \reals^v$ and $t_0 \leq r_i < r_{i + 1} \leq t_n$ the log-signature of $X$ over the interval $[r_i, r_{i+1}]$ was defined to be a particular collection of statistics $\logsig^N_{r_i, r_{i+1}}(X) \in \reals^{\beta(v, N)}$; specifically those statistics that best describe how $X$ drives the CDE equation (\ref{eq:kidger_cde}).

\subsection{The Rough Hidden State Update}

Recall how the Neural CDE formulation of equation (\ref{eq:kidger_ncde}) was solved via equations (\ref{eq:kidger_cde_evaluation}), (\ref{eq:kidger_g_X}). For the rough approach we begin by replacing (\ref{eq:kidger_g_X}) with the piecewise
\begin{equation}
    \widehat{g}_{\theta, X}(Z, s) = \widehat{f}_\theta(Z) \frac{\logsig_{r_i, r_{i+1}}^N(X)}{r_{i+1} - r_i} \quad \text{for } s \in [r_i, r_{i + 1}),
    \label{eq:log_ode_g}
\end{equation}
where $\widehat{f}_\theta \colon \reals^w \to \reals^{w \times \beta(v, N)}$ is an arbitrary neural network, and the right hand side denotes a matrix-vector product between $\widehat{f}_\theta$ and the log-signature. Equation (\ref{eq:kidger_cde_evaluation}) then becomes
\begin{equation}
    Z_t = Z_{t_0} + \int^{t}_{t_0} \widehat{g}_{\theta, X}(Z_s, s) \dby s.
    \label{eq:log_ode_ncde}
\end{equation}
This may now be solved as a (neural) ODE using standard ODE solvers.

We give an overview of this process in figure \ref{fig:ncde_plots}. The left hand side represents a single step method, as in the existing approach to Neural CDEs. The right hand side depicts a rough approach that takes steps larger than the discretisation of the data in exchange for additional terms of the log-signature.

\subsection{Neural RDEs Generalise Neural CDEs}\label{section:generalise}
Suppose we happened to choose $r_i = t_i$ and $r_{i + 1} = t_{i + 1}$. Then the log-signature term is 
\begin{equation*}
    \frac{\logsig_{t_i, t_{i+1}}^N(X)}{t_{i+1} - t_i}
\end{equation*}
Recall that the depth 1 log-signature is just the increment of the path over the interval. So this becomes
\begin{equation*}
    \frac{\Delta X_{[t_i, t_{i+1}]}}{t_{i+1} - t_i} = \frac{\dby X^{\mathrm{linear}}}{\dby t}(s) \quad \text{for } s \in [t_i, t_{i + 1}),
\end{equation*}
that is to say the same as obtained via the original method if using linear interpolation. In this way the Neural RDE approach generalises the existing Neural CDE approach.
    
\subsection{Discussion}\label{subsec:tradeoff}

\paragraph{Length/Channel Trade-Off} 
The sequence of log-signatures is now of length $m$, which was chosen to be much smaller than $n$. As such, it is much more slowly varying over the interval $[t_0, t_n]$ than the original data, which was of length $n$. The differential equation it drives is better behaved, and so larger integration steps may be used in the numerical solver. This is the source of the speed-ups of this method; we observe typical speed-ups by a factor of about 10.


\paragraph{Memory Efficiency} Long sequences need large amounts of memory to perform backpropagation-through-time (BPTT). As with the original Neural CDEs, the log-ODE approach supports memory-efficient backpropagation via the adjoint equations. If the vector field $f_\theta$ requires $\mathcal{O}(H)$ memory, and the time series is of total length $T$, then backpropagating through the solver requires $\mathcal{O}(HT)$ memory whilst the adjoint method requires only $\mathcal{O}(H + T)$; see \citet{kidger2020neural}.

\paragraph{The Log-signature as a Preprocessing Step} When training a model in practice, the log-signatures need only be computed once and thus the computation can be performed as part of data preprocessing. Log-signatures can also be easily computed in an online fashion, making the model suitable for such problems. 

\paragraph{Structure of $\widehat{f}$} The description here aligns with the log-ODE scheme described in equation (\ref{eq:log-ode}). There is one discrepancy: we do not attempt to model the specific structure of $\widehat{f}$. This is in principle possible, but is computationally expensive. Instead, we model $\widehat{f}$ as a neural network directly. This need \emph{not} necessarily exhibit the requisite structure, but as neural networks are universal approximators \citep{pinkus, deepandnarrow} then this approach is at least as general from a modelling perspective.

\paragraph{Ease of Implementation}
This method is straightforward to implement using pre-existing tools.

There are standard libraries available for computing the log-signature transform: we use Signatory \citep{signatory}. As equation (\ref{eq:log_ode_ncde}) is an ODE, it may be solved directly using tools such as \texttt{torchdiffeq} \citep{torchdiffeq}.

As an alternative, we note that the form of equation (\ref{eq:log_ode_g}) is that of equation (\ref{eq:kidger_g_X}), with the driving path taken to be piecewise linear in log-signature space. Computation of the log-signatures can therefore be considered as a preprocessing step, producing a sequence of log-signatures. From this we may construct a path in log-signature space, and apply existing tools for neural CDEs. (Rather than tools for neural ODEs.) This idea is summarised in figure \ref{fig:ncde_plots}. We make this approach available in the [redacted] open source project.



\paragraph{Applications} In principle, a Neural RDE may be applied to solve any Neural CDE problem. However, we typically observe limited benefit on relatively short time series: the original Neural CDE formulation works well enough, and there is little room to see either speed or loss/accuracy improvements via this approach.

The situation changes for long time series. Here, the existing approach struggles as the length of the time series grows. Performance worsens, and speed drops due to the sheer number of forward evaluations. This is the same behaviour as for RNNs.

Now, the reduction in length (from $n$ to $m \ll n$) is highly beneficial. Moreover, the compression performed by the log-signature is also of benefit: closely-sampled points will be typically be strongly correlated, and there is little to be gained by treating them all individually.

In addition, there are two advantages shared by both Neural CDEs and Neural RDEs, that make them suitable for long time series. The first is the sharply reduced memory requirements of the adjoint method. For example (chosen arbitrarily without cherry-picking) in one experiment we see a reduction in memory usage from $3.6$GB to just $47$MB.

The second is that as both operate in continuous time, the steps in the numerical solver may be decoupled from the sampling rate of the data: steps are taken with respect to the complexity of the data, not just its sampling rate. In particular a slowly-varying but densely-sampled path would still be fast without requiring many integration steps.

\paragraph{The Depth and Step Hyperparameters}To solve a Neural RDE accurately via the log-ODE method, we should be prepared to take the depth $N$ suitably large, or the intervals $r_{i + 1} - r_i$ suitably small. Accomplishing this would often require that they are taken relatively large or relatively small, respectively. Instead, we treat these as hyperparameters. This makes use of the log-ODE method a modelling choice rather than an implementation detail.

Increasing step size will lead to faster (but less informative) training by reducing the number of operations in the forward pass. Increasing depth will lead to slower (but more informative) training, as more information about each local interval is used in each update.

\section{Experiments} \label{sec:experiments}
We run experiments applying Neural RDEs to four real-world datasets. Every problem was chosen for its long length. The lengths are sufficiently long that adjoint-based backpropagation \citep{neural2018ode} was often needed simply to avoid running out of memory at any reasonable batch size. Every problem is regularly sampled, so we take $t_i = i$.

Recall that the Neural RDE approach features two hyperparameters, corresponding to log-signature depth and step size. Good choices will turn out to have a dramatic positive effect on performance. Accordingly for every experiment we run Neural RDEs for all depths in $N=2,3$ and all step sizes in $2,4,8,16,32,64,128,256,512,1024$. Depth 1 and step 1 are not considered as both reduce onto the Neural CDE model, as discussed in section \ref{section:generalise}. In practice, when choosing a final model, one would choose that with depth and step values that minimise the validation loss, as in any hyperparamter value selection.

We compare against two baseline models. The first is a Neural CDE; as the model we are extending then comparisons to this are our primary concern. For context we also additionally include a baseline against the ODE-RNN introduced in \citet{rubanova2019latent}. For both of these models, we also run experiments on the full range of step sizes described above.

For the Neural CDE model, increased step sizes correspond to na{\"i}ve subsampling of the data (in accordance with section \ref{section:generalise}). For the ODE-RNN model, we instead fold the time dimension into the feature dimension, so that at each step the ODE-RNN model sees several adjacent time points. This represents an alternate technique for dealing with long time series, so as to provide a reasonable benchmark.



For each model, and each hyperparameter combination, we run the experiment three times and report the mean and standard deviation of the test metrics. We additionally report mean training times and memory usages. 


Precise details of hyperparameter selection, optimisers, normalisation, and so on can be found in Appendix \ref{apx:experiments}. For brevity, we provide results for only some of the step sizes here. The full results are described in Appendix \ref{apx:results}.

\subsection{Classifying EigenWorms}
\begin{table}[t]
    \small
    \setlength{\tabcolsep}{4.pt}
    \begin{center}
        \begin{tabular}{ccccc}
        \toprule
        \textbf{Model} & \textbf{Step} & \textbf{Accuracy (\%)} & \textbf{Time (Hrs)} & \textbf{Mem (Mb)} \\
        \midrule
          & 1    &  -- & -- & -- \\
        ODE-RNN  & 4    &   35.0 $\pm$ 1.5 &           0.8 &        3629.3 \\
        (folded) & 32   &   32.5 $\pm$ 1.5 &           0.1 &         532.2 \\
          & 128  &   47.9 $\pm$ 5.3 &           0.0 &         200.8 \\
          
        \hdashline\noalign{\vskip 0.5ex}
          & 1    &  62.4 $\pm$ 12.1 &          22.0 &         176.5 \\
        \multirow{2}{*}{NCDE}    & 4    &  66.7 $\pm$ 11.8 &           5.5 &          46.6 \\
          & 32   &  64.1 $\pm$ 14.3 &           0.5 &           8.0 \\
          & 128  &   48.7 $\pm$ 2.6 &           0.1 &           3.9 \\
          
        \midrule
        \multirow{3}{*}{\begin{tabular}{c} NRDE\\(depth 2)\end{tabular}}  & 4    &   \textbf{   83.8 $\pm$ 3.0$^*$} &           2.4 &         180.0 \\
        & 32   &  67.5 $\pm$ 12.1 &           0.7 &          28.1 \\
          & 128  &   \textbf{76.1 $\pm$ 5.9} &           0.2 &           7.8 \\
          
        \hdashline\noalign{\vskip 0.5ex}
        \multirow{3}{*}{\begin{tabular}{c} NRDE\\(depth 3)\end{tabular}}  & 4    &   76.9 $\pm$ 9.2 &           2.8 &         856.8 \\
         & 32   &   \textbf{75.2 $\pm$ 3.0} &           0.6 &         134.7 \\
          & 128  &   68.4 $\pm$ 8.2 &           0.1 &          53.3 \\
        \bottomrule
        \end{tabular}
    \end{center}
    \caption{EigenWorms dataset: mean $\pm$ standard deviation of test set accuracy measured over three repeats. Also reported are the mean memory usage and training time. For all models a variety of step sizes are considered. For the Neural RDE we additionally investigate varying depths. (Recalling that the NCDE is a depth-1 NRDE.) `--' denotes that the model could not be run within GPU memory. Bold denotes the best model score for a given step size, and $^*$ denotes that the score was the best achieved over all models and step sizes.}
    \label{tab:eigenworms}
\end{table}

Our first example uses the EigenWorms dataset from the UEA archive from \citet{bagnall16bakeoff}. This consists of time series of length 17\,984 and 6 channels (including time), corresponding to the movement of a roundworm. The goal is to classify each worm as either wild-type or one of four mutant-type classes.


Results are shown in Table \ref{tab:eigenworms}. We begin by seeing that the step-1 Neural CDE model takes roughly a day to train. Switching to Neural RDEs speeds this up by an order of magnitude, to roughly two hours. Moreover doing so dramatically improves accuracy, by up to $17\%$, reflecting the classical difficulty of learning from long time series.

Meanwhile na{\"i}ve subsampling approaches for the Neural CDE method only achieve speed-ups without performance improvements. The folded ODE-RNN model performs poorly, attaining the worst score for any step size whilst imposing a significantly higher memory burden. 



Results across all step sizes may be found in Appendix \ref{apx:results}.

\subsection{Estimating Vitals Signs from PPG and ECG data}
\begin{table*}[t]
    \small
    \begin{center}
        \begin{tabular}{cccccccccc}
        \toprule
        \multirow{2}{*}{\textbf{Model}} &
        &
        \multirow{2}{*}{\textbf{Step}} & \multicolumn{3}{c}{$\mathbf{L^2}$} & \multicolumn{3}{c}{\textbf{Time (Hrs)}} & \multirow{2}{*}{\textbf{Memory (Mb)}} \\
        \cmidrule(lr){4-6} \cmidrule(lr){7-9}
        & & & RR & HR & SpO$_2$ & RR & HR & SpO$_2$ & \\
         
        \midrule
        
         & & 1   &  -- & 13.06  $\pm$  0.0 & -- & -- & 10.5 & -- & 3653.0 \\
        ODE-RNN (folded) & \multirow{2}{*}{} & 8   & 2.47 $\pm$ 0.35 & 13.06 $\pm$ 0.00 & 3.3 $\pm$ 0.00 & 1.5 & 1.2 & 0.9 & 917.2 \\
         & & 128  &  1.62 $\pm$ 0.07 &  13.06 $\pm$ 0.00 &    3.3 $\pm$ 0.00 &           0.2 &           0.1 &           0.1 &               81.9 \\
         & & 512  &  1.66 $\pm$ 0.06 &   6.75 $\pm$ 0.9 &  1.98 $\pm$ 0.31 &           0.0 &           0.1 &           0.1 &               40.4 \\
         
        \hdashline\noalign{\vskip 0.5ex}
        
        & & 1   &  2.79 $\pm$ 0.04 &   9.82 $\pm$ 0.34 &  2.83 $\pm$ 0.27 &          23.8 &          22.1 &          28.1 &               56.5 \\
        \multirow{2}{*}{NCDE} & \multirow{2}{*}{} & 8   &   2.80 $\pm$ 0.06 &  10.72 $\pm$ 0.24 &  3.43 $\pm$ 0.17 &           3.0 &           2.6 &           4.8 &               14.3 \\
        &  & 128 &  2.64 $\pm$ 0.18 &  11.98 $\pm$ 0.37 &  2.86 $\pm$ 0.04 &           0.2 &           0.2 &           0.3 &                8.7 \\
        &  & 512 &  2.53 $\pm$ 0.03 &  12.22 $\pm$ 0.11 &  2.98 $\pm$ 0.04 &           0.1 &           0.0 &           0.1 &                8.4 \\
          
        \midrule
        & & 8   &  2.63 $\pm$ 0.12 &   8.63 $\pm$ 0.24 &  2.88 $\pm$ 0.15 &           2.1 &           3.4 &           3.3 &               21.8 \\
        NRDE (depth 2) &  & 128  &  1.86 $\pm$ 0.03 &   6.77 $\pm$ 0.42 &  1.95 $\pm$ 0.18 &           0.3 &           0.4 &           0.7 &               10.9 \\
        &  & 512 &  1.81 $\pm$ 0.02 &   5.05 $\pm$ 0.23 &  2.17 $\pm$ 0.18 &           0.1 &           0.2 &           0.4 &               10.3 \\
        \hdashline\noalign{\vskip 0.5ex}
        & & 8   &  \textbf{2.42 $\pm$ 0.19} &    \textbf{7.67 $\pm$ 0.40} &  \textbf{2.55 $\pm$ 0.13} &           2.9 &           3.2 &           3.1 &               43.3 \\
        NRDE (depth 3) &   & 128  &  \textbf{1.51 $\pm$ 0.08} &   \textbf{$\;$2.97 $\pm$ 0.45}$^*$ &  \textbf{1.37 $\pm$ 0.22} &           0.5 &           1.7 &           1.7 &               17.3 \\
        &  & 512  &  \textbf{$\;$1.49 $\pm$ 0.08}$^*$ &   \textbf{3.46 $\pm$ 0.13} &  \textbf{$\;$1.29 $\pm$ 0.15}$^*$ &           0.3 &           0.4 &           0.4 &               15.4 \\
        \bottomrule
        \end{tabular}
    \end{center}
    \caption{The three experiments on BIDMC datasets: mean $\pm$ standard deviation of test set $L^2$ loss, measured over three repeats, over each of three different vital signs prediction tasks (RR, HR, SpO$_2$). Also reported are the memory usage and training time. Only mean times are shown for space. For all models a variety of step sizes are considered. For the Neural RDE we additionally investigate varying depths. (Recalling that the NCDE is a depth-1 NRDE.) `--' denotes that the model could not be run within GPU memory. Bold denotes the best model score for a given step size, and $^*$ denotes that the score was the best achieved over all models and step sizes.}
    \label{tab:bidmc}
\end{table*}
\begin{figure*}
    \centering
    \includegraphics[width=\textwidth]{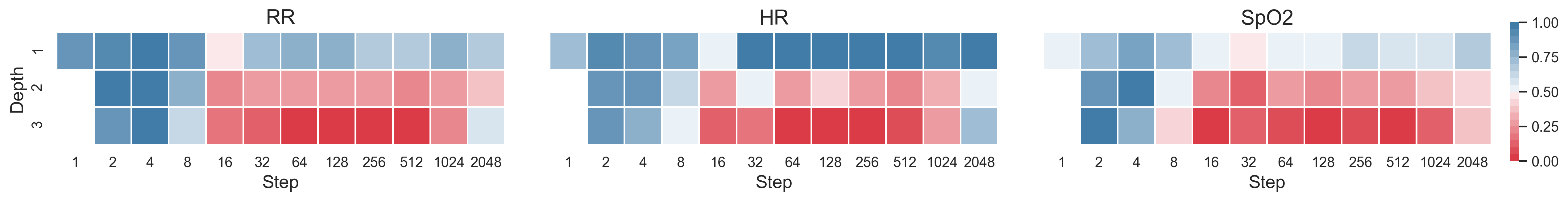}
    \caption{Heatmap depicting normalised losses on the three BIDMC datasets for differing step sizes and depths. We can see that the point of lowest MSE (deepest red) has step $>1$ and depth $>1$, and that performance worsens for very long steps. This represents the depth/step tradeoff for long length time series. }
    \label{fig:bidmc_heat}
\end{figure*}

Next we consider three separate problems, using data from the TSR archive \citep{MonashTSRegressionArchive}, coming originally from the Beth Israel Deaconess Medical Centre (BIDMC).

We aim to predict a person's respiratory rate (RR), their heart rate (HR), or their oxygen saturation (SpO2) at the end of the sample, having observed PPG and ECG data over the length of the sample. The data is sampled at 125Hz and each series has length 4\,000. There are 3 channels (including time). We evaluate performance with the $L^2$ loss.

The results are shown in table \ref{tab:bidmc}.

We find that the depth $3$ Neural RDE is the top performer for every task at every step size, reducing test loss by $30$--$59\%$ versus the Neural CDE. Moreover, it does so with roughly an order of magnitude less training time.

We attribute the improved test loss to the Neural RDE model being better able to learn long-term dependencies due to the reduced sequence length. Note that the performance of the rough models actually improves as the step size is increased. This is in contrast to Neural CDE, which sees a degradation in performance.

The ODE-RNN model, besides using significantly more memory, struggles to train effectively when the sequence length is long. Training improves as the sequence size is shortened, but still produces results substantially worse than those achieved by the Neural RDE.

As a visual summary of these results, including the full range of step sizes, we also provide heatmaps in Figure \ref{fig:bidmc_heat}.

The full results across the full range of step sizes may be found in Appendix \ref{apx:results}.

\section{Limitations}

\paragraph{Number of hyperparameters} Two new hyperparameters -- truncation depth and step size -- with substantial effects on training time and memory usage must now also be tuned.

\paragraph{Number of input channels} The log-ODE method is most feasible with few input channels, as the number of log-signature channels $\beta(v, N)$ grows exponentially in $v$. For larger $v$ then the available parallelism may become saturated.

\section{Related Work}
CNNs and Transformers have been shown to offer improvements over RNNs for modelling long-term dependencies \citep{bai2018empirical, li2019enhancing}, although the latter in particular have typically focused on language modelling. On a more practical note, Transformers are famously $\mathcal{O}(L^2)$ in the length of the time series $L$. Several approaches have been introduced to reduce this, for example \citet{li2019enhancing} reduce this to $\mathcal{O}(L(\log L)^2)$. Extensions specifically to long sequences do exist \citep{sourkov2018igloo}, but again these typically focus on language modelling rather than multivariate time series data. 

There has also been some work on long time series for classic RNN (GRU/LSTM) models.

\citet{wisdom2016full, jing2019gated} show that unitary or orthogonal RNNs can mitigate the vanishing/exploding gradients problem. However, they are expensive to train due to the need to compute a matrix inversion at each training step. \citet{chang2017dilated} introduce dilated RNNs with skip connections between RNN states, which help improve training speed and learning of long-term dependencies. \citet{campos2017skip} introduce the `Skip-RNN' model, which extend the RNN by adding an additional learnt component that skips state updates. \citet{li2018independently} introduce the `IndRNN' model, with particular structure tailored to learning long time series.

One meaningful comparison is to hierarchical subsampling as in \citet{graves2012supervised, de2015survey}. There the data is split into windows, an RNN is run over each window, and then an additional RNN is run over the first RNN's outputs; we may describe this as an RNN/RNN pair. \citet{liao2019learning} then perform the equivalent operation with a log-signature/RNN pair. In this context, our use of log-ODE method is analogous to an log-signature/NCDE pair.

In comparison to \citet{liao2019learning}, this means moving from an inspired choice of pre-processing to an actual implementation of the log-ODE method. In doing so the differential equation structure is preserved. Moreover this takes advantage of the synergy between log-signatures (which extract statistics on how data drives differential equations), and the controlled differential equation it then drives. Broadly speaking these connections are natural: at least within the signature/CDE/rough path community, it is a well-known but poorly-published fact that RNNs, (log-)signatures, and (Neural) CDEs are all related; see for example \citet{kidger2020neural} for a little exposition on this.

\citet{de2019gru, lechner2020learning} amongst others consider continuous time modifications to GRUs and LSTMs, improving the learning of long-term dependencies.

\citet{lmu, hippo} consider links with ODEs and approximation theory, with the goal of improving the long-term memory capacity of RNNs. Given the differential equation structure both they and we consider, a hybridisation of these techniques seems like a promising line of future inquiry.

\section{Conclusion}
We have introduced \textit{neural rough differential equations} as an approach to continuous-time time series modelling. These extend Neural CDEs, driving the hidden state not by point evaluations but by interval summarisations of the underlying time series or control path. Neural RDEs may still be solved via ODE methods, and thus retain both adjoint backpropagation and continuous dynamics. As they additionally reduce the effective length of the control path, we observe substantial practical benefits in applying Neural RDEs to long time series. In this regime we report significant training speed-ups, model performance improvements, and reduced memory requirements, on problems of length up to 17\,000.

\section*{Acknowledgements} JM was supported by the EPSRC grant EP/L015803/1 in collaboration with Iterex Therapuetics. PK was supported by the EPSRC grant EP/L015811/1. CS was supported by the EPSRC grant EP/R513295/1. JF was supported by the EPSRC grant EP/N509711/1. JM, CS, PK, JF were supported by the Alan Turing Institute under the EPSRC grant EP/N510129/1.

\bibliography{references}
\bibliographystyle{style/icml2021}


\newpage
\onecolumn
\appendix
\begin{center}
	\huge Supplementary material
\end{center}
	
In sections \ref{apx:logode} and \ref{apx:logodeconv}, we give a more thorough introduction to solving CDEs via the log-ODE method.

In section \ref{apx:experiments} we discuss the experimental details such as the choice of network structure, computing infrastructure and hyperparameter selection approach.

In section \ref{apx:results} we give a full breakdown of every experimental result.

\section{An introduction to the log-ODE method for controlled differential equations}\label{apx:logode}

The log-ODE method is an effective method for approximating the controlled differential equation:
\begin{align}
\dby Y_t & = f(Y_t)\,\dby X_t,
\label{eq:cde_for_logode}\\[3pt]
Y_0 & = \xi,\nonumber
\end{align}
where $X : [0,T]\rightarrow \R^d$ has finite length, $\xi\in\R^n$ and $f : \R^n \rightarrow L(\R^d, \R^n)$ is a function with certain smoothness assumptions so that the CDE (\ref{eq:cde_for_logode}) is well posed. Throughout these appendices, $L(U, V)$ denotes the space of linear maps between the vector spaces $U$ and $V$. In rough path theory, the function $f$ is referred to as the ``vector field'' of (\ref{eq:cde_for_logode}) and usually assumed to have \textup{Lip(}$\boldsymbol{\gamma}$\textup{)} regularity
(see definition 10.2 in \citet{friz2010multidimensional}). In this section, we assume one of the below conditions on the vector field:
\begin{enumerate}
\item $f$ is bounded and has $N$ bounded derivatives.
\item $f$ is linear.
\end{enumerate}

In order to define the log-ODE method, we will first consider the tensor algebra and path signature.\smallbreak

\begin{definition}\label{def:tensoralgebras}
We say that $T\big(\R^d\big):= \R \oplus \R^d \oplus (\R^d)^{\otimes 2}\oplus\cdots$ is the \emph{tensor algebra} of $\R^d$ and $T\big(\big(\R^d\big)\big):=\big\{\boldsymbol{a} = \big(a_{0}, a_{1}, \cdots\big) : a_{k}\in\big(\R^d\big)^{\otimes k}\,\,\forall k\geq 0\big\}$ is the set of formal series of tensors of $\R^d$.
Moreover, $T\big(\R^d\big)$ and $T\big(\big(\R^d\big)\big)$ can be endowed with the operations
of addition and multiplication.
Given $\boldsymbol{a} = (a_{0}, a_{1}, \cdots)$ and $\boldsymbol{b} = (b_{0}, b_{1}, \cdots)$, we have
\begin{align}
\boldsymbol{a} + \boldsymbol{b} & =\big(a_{0} + b_{0}, a_{1} + b_{1}, \cdots\big),\label{eq:tensoradd}\\
\boldsymbol{a} \otimes \boldsymbol{b} & =\big(c_{0}, c_{1}, c_{2}, \cdots\big),\label{eq:tensormult}
\end{align}
where for $n\geq 0$, the $n$-th term $c_{n}\in\big(\R^d\big)^{\otimes n}$ can be written as
\begin{equation}\label{eq:tproduct2}
c_{n} := \sum_{k=0}^{n}a_{k}\otimes b_{n-k}.
\end{equation}
The use of $\otimes$ in equation (\ref{eq:tproduct2}) denotes the usual tensor product. The use of $\otimes$ in equation (\ref{eq:tensormult}) is \textit{also} referred to as the ``tensor product'': when precisely one $a_i$ and precisely one $b_i$ are nonzero then it reduces to the usual tensor product; equation (\ref{eq:tensormult}) is a generalisation.
\end{definition}\medbreak

\begin{definition}\label{def:signature_appendix}
The \emph{signature} of a finite length path $X : [0,T]\rightarrow\R^d$ over
the interval $[s,t]$ is defined as the following collection of iterated (Riemann--Stieltjes) integrals:
\begin{equation}\label{eq:fullsig}
S_{s,t}\big(X\big) := \Big(1\hspace{0.5mm}, X_{s,t}^{(1)}, X_{s,t}^{(2)}, X_{s,t}^{(3)},\ldots\Big)\in T\big(\big(\R^d\big)\big),
\end{equation}
where for $n\geq 1$,
\begin{equation}
X_{s,t}^{(n)} := \idotsint\displaylimits_{s<u_{1}<\cdots<u_{n}<t}\dby X_{u_{1}}\otimes\cdots\otimes \dby X_{u_{n}}\in\big(\R^d\big)^{\otimes n}.\nonumber
\end{equation}
Similarly, we can define the depth-N (or truncated) signature of the path $X$ on $[s,t]$ as
\begin{equation}\label{eq:truncsig}
S_{s,t}^{N}\big(X\big) := \Bigg(\hspace{0.5mm}1\hspace{0.5mm}, X_{s,t}^{(1)}, X_{s,t}^{(2)},\ldots,X_{s,t}^{(N)}\Bigg)\in T^N\big(\R^d\big),
\end{equation}
where $T^N\big(\R^d\big):= \R \oplus \R^d \oplus (\R^d)^{\otimes 2}\oplus\cdots\oplus (\R^d)^{\otimes N}$ denotes the truncated tensor algebra.
\end{definition}

The (truncated) signature provides a natural feature set that describes the effects a path $X$ has on systems that can be modelled by (\ref{eq:cde_for_logode}). That said, defining the log-ODE method actually requires the so-called ``log-signature'' which efficiently encodes the same integral information as the signature. The log-signature is obtained from the path's signature by removing certain algebraic redundancies, such as
\begin{equation}
\int_0^t\int_0^s \dby X_u^{i} \dby X_s^{j} + \int_0^t\int_0^s \dby X_u^{j} \dby X_s^{i} = X_t^{i}X_t^{j},
\nonumber
\end{equation}
for $i,j\in\{1,\cdots, d\}$, which follows by the integration-by-parts formula. To this end, we will define the logarithm map on the depth-$N$ truncated tensor algebra $T^N\big(\R^d\big):= \R \oplus \R^d\oplus\cdots\oplus (\R^d)^{\otimes N}$.

\begin{definition}[The logarithm of a formal series]\label{def:tensor_log} For $\boldsymbol{a} = (a_{0}, a_{1}, \cdots) \in T\big(\big(\R^{d}\big)\big)$ with $a_{0} > 0$, define $\log(\boldsymbol{a})$ to be the element of $T\big(\big(\R^{d}\big)\big)$ given by the following series:
\begin{align}
\log(\boldsymbol{a}) & := \log(a_{0}) + \sum_{n=1}^{\infty}\frac{(-1)^{n}}{n}\bigg(\boldsymbol{1} - \frac{\boldsymbol{a}}{a_{0}}\bigg)^{\otimes n},\label{eq:fulltensorlogseries}
\end{align}
where $\boldsymbol{1} = (1, 0, \cdots)$ is the unit element of $T\big(\big(\R^{d}\big)\big)$ and $\log(a_{0})$ is viewed as $\log(a_{0})\boldsymbol{1}$.
\end{definition}
\begin{definition}[The logarithm of a truncated series] For $\boldsymbol{a} = (a_{0}, a_{1}, \cdots, a_N) \in T\big(\big(\R^{d}\big)\big)$ with $a_{0} > 0$, define $\log^N(\boldsymbol{a})$ to be the element of $T^N\big(\R^d\big)$ defined from the logarithm map (\ref{eq:fulltensorlogseries}) as
\begin{align}
\log^N(\boldsymbol{a})  & := P_N\big(\log(\boldsymbol{\widetilde{a}})\big),\label{eq:trunctensorlogseries}
\end{align}
where $\boldsymbol{\widetilde{a}} := (a_{0}, a_{1}, \cdots, a_N, 0, \cdots)\in T\big(\big(\R^{d}\big)\big)$ and $P_N$ denotes the standard projection map from $T\big(\big(\R^{d}\big)\big)$ onto $T^N\big(\R^d\big)$.
\end{definition}

\begin{definition}\label{def:logsig_appendix}The \emph{log-signature} of a finite length path $X : [0,T]\rightarrow\R^d$ over the interval $[s,t]$ is defined as $\logsig_{s,t}(X) := \log(S_{s,t}(X))$, where $S_{s,t}(X)$ denotes the path signature of $X$ given by Definition \ref{def:signature_appendix}.
Likewise, the depth-N (or truncated) log-signature of $X$ is defined for each $N\geq 1$ as $\logsig_{s,t}^N(X) := \log^N(S_{s,t}^N(X))$.
\end{definition}

In this section, we view each $\logsig_{s,t}^N(X)$ as an element of $T^N\big(\R^d\big)$ to simplify the definition of the log-ODE method. That said, this is equivalent to the definition used in the main body of the paper, which defines the log-signature as a map from $X \colon [0, T] \to \reals^d$ to $\reals^{\beta(d, N)}$. This corresponds to the interpretation of a log-signature as an element of a certain free Lie algebra (see, for example, \citet{roughpath2007notes, reizenstein2017logsig} for details). The exact form of $\beta(d, N)$ is given by
\begin{equation*}
    \beta(d, N) = \sum_{k = 1}^N \frac{1}{k} \sum_{i | k} \mu\left(\frac{k}{i}\right) d^i
    \label{eq:logsig-dim}
\end{equation*}
with $\mu$ the M{\"o}bius function. The precise order of this remains an open question.

The final ingredient we use to define the log-ODE method are the derivatives of the vector field $f$.
It is worth noting that these derivatives also naturally appear in the Taylor expansion of (\ref{eq:cde_for_logode}).

\begin{definition}[Vector field derivatives]\label{def:vect_derivative} We define $f^{\circ k} : \R^n\rightarrow L((\R^d)^{\otimes k}, \R^n)$ recursively by
\begin{align*}
f^{\circ (0)}(y) & := y,\\
f^{\circ (1)}(y) & := f(y),\\
f^{\circ (k + 1)}(y) & := D\big(f^{\circ k}\big)(y)f(y),
\end{align*}
for $y\in\R^n$, where $D\big(f^{\circ k}\big)$ denotes the Fr\'{e}chet derivative of $f^{\circ k}$.
\end{definition}

Using these definitions, we can describe two closely related numerical methods for the CDE (\ref{eq:cde_for_logode}).\smallbreak

\begin{definition}[\textbf{The Taylor method}] Given the CDE (\ref{eq:cde_for_logode}), we can use the path signature of $X$ to approximate the solution $Y$ on an interval $[s,t]$ via its truncated Taylor expansion. That is, we use
\begin{equation}
\mathrm{Taylor}(Y_s, f, S_{s,t}^N(X)) := \sum_{k=0}^N f^{\circ k}(Y_s)\pi_k \big(S_{s,t}^N(X)\big),
\label{eq:taylor_def}
\end{equation}
as an approximation for $Y_t$ where each $\pi_k : T^N(\R^d)\rightarrow (\R^d)^{\otimes k}$ is the projection map onto $\big(\R^d\big)^{\otimes k}$.
\end{definition}\medbreak

\begin{definition}[\textbf{The Log-ODE method}]\label{def:logode} Using the Taylor method (\ref{eq:taylor_def}), we can define the function $\widehat{f} : \R^n \rightarrow L(T^N(\R^d), \R^n)$ by $\widehat{f}(z) := \mathrm{Taylor}(z, f, \cdot\m)$. By applying $\widehat{f}$ to the truncated log-signature of the path $X$ over an interval $[s,t]$, we can define the following ODE on $[0,1]$
\begin{align}
\frac{\dby z}{\dby u} & = \widehat{f}(z)\logsig_{s,t}^N(X),\label{eq:first_logode_def}\\[3pt]
z(0) & = Y_s.\nonumber
\end{align}
Then the log-ODE approximation of $Y_t$ $($given $Y_s$ and $\logsig_{s,t}^N(X))$ is defined as
\begin{equation}
\mathrm{LogODE}(Y_s, f, \logsig_{s,t}^N(X)) := z(1).
\label{eq:second_logode_def}
\end{equation}
\end{definition}
\begin{remark}
Our assumptions of $f$ ensure that $z\mapsto\widehat{f}(z)\logsig_{s,t}^N(X)$ is either globally bounded and Lipschitz continuous or linear. Hence both the Taylor and log-ODE methods are well defined.
\end{remark}
\begin{remark}
It is well known that the log-signature of a path $X$ lies in a certain free Lie algebra (this is detailed in section 2.2.4 of \citet{roughpath2007notes}). Furthermore, it is also a theorem that the Lie bracket of two vector fields is itself a vector field which doesn't depend on choices of basis.
By expressing $\logsig_{s,t}^N(X)$ using a basis of the free Lie algebra, it can be shown that only the vector field $f$ and its (iterated) Lie brackets are required to construct the log-ODE vector field $\widehat{f}(z)\logsig_{s,t}^N(X)$. In particular, this leads to our construction of the log-ODE (\ref{eq:log-ode}) using the Lyndon basis of the free Lie algebra (see \cite{reizenstein2017logsig} for a precise description of the Lyndon basis). We direct the reader to \citet{lyons2014streams} and \citet{logode2014estimate} for further details on this Lie theory.
\end{remark}


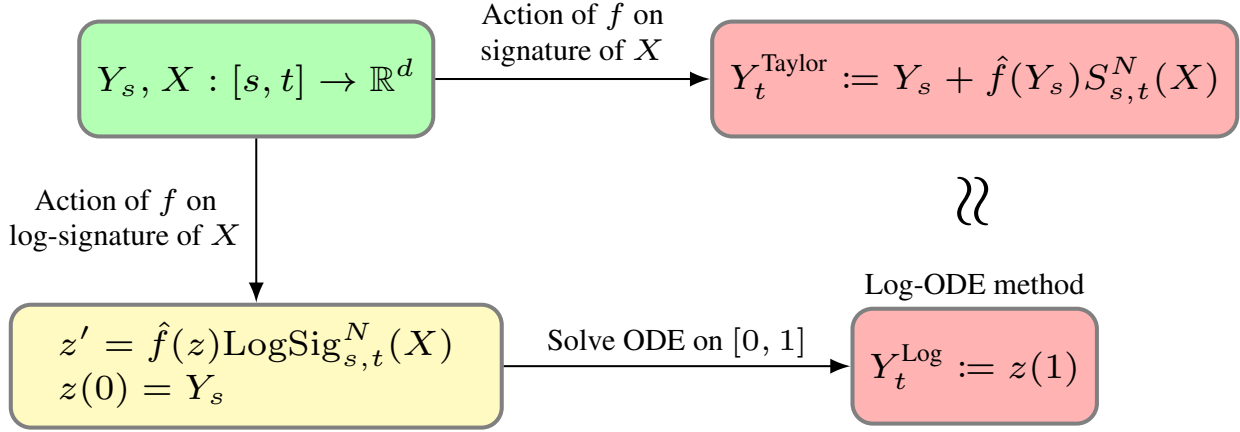
\begin{figure*}
    \centering
    \resizebox{\textwidth}{!}{%
    
    \begin{tikzpicture}
        \node (initial) at (0, 0) [black, thick, rounded corners, rectangle, draw=black!50, fill=green!30, minimum height=0.8cm] {\scriptsize $Y_s,$ $X:[s, t] \rightarrow \mathbb{R}^d$};
        \node (taylor) at (5, 0) [black, thick, rounded corners, rectangle, draw=black!50, fill=red!30, minimum height=0.8cm] {\scriptsize $Y_t^{\text{Taylor}} \coloneqq Y_s + \hat{f}(Y_s)S^N_{s, t}(X)$};
        \node (logsig) at (0, -2) [black, thick, rounded corners, rectangle, draw=black!50, fill=yellow!30, minimum height=0.8cm] {
            \scriptsize 
            \begin{tabular}{l}
                $z' = \hat{f}(z)\mathrm{LogSig}^N_{s, t}(X)$\\
                $z(0)=Y_s$    
            \end{tabular}
        };
        \node (logode) at (5, -2) [black, thick, rounded corners, rectangle, draw=black!50, fill=red!30, minimum height=0.8cm, label={[shift={(0, -0.08)}]\tiny Log-ODE method}] {\scriptsize $Y^{\text{Log}}_t \coloneqq z(1)$};
        \node (approx) at (5, -0.85) [rotate=90] {\Large $\approx$};
        
        \draw[->] (initial) -- (taylor) node[midway, above, yshift=-0.4ex] {\tiny \begin{tabular}{c}Action of $f$ on\\signature of $X$\end{tabular}};
        \draw[->] (initial) -- (logsig) node[midway, left, xshift=1.5ex] {\tiny \begin{tabular}{c}Action of $f$ on\\log-signature of $X$\end{tabular}};
        \draw[->] (logsig) -- (logode) node[midway, above, yshift=-0.4ex] {\tiny Solve ODE on $[0, 1]$};

    \end{tikzpicture}
}
    \caption{Illustration of the log-ODE and Taylor methods for controlled differential equations.} 
    \label{fig:logode_taylor}
\end{figure*}

To illustrate the log-ODE method, we give two examples:
\begin{example}[The ``increment-only'' log-ODE method]
When $N = 1$, the ODE (\ref{eq:first_logode_def}) becomes
\begin{align*}
\frac{\dby z}{\dby u} & = f(z)X_{s,t},\\[3pt]
z(0) & = Y_s.
\end{align*}
Therefore we see that this ``increment-only'' log-ODE method is equivalent to driving the original CDE (\ref{eq:cde_for_logode}) by a piecewise linear approximation of the control path $X$. This is a classical approach
for stochastic differential equations (i.e. when $X_t = (t, W_t)$ with $W$ denoting a Brownian motion) and is an example of a Wong-Zakai approximation (see \citet{wongzakai1965} for further details).
\end{example}

\begin{example}[An application for SDE simulation] Consider the following affine SDE,
\begin{align}\label{eq:IGBM}
\dby Y_t & = a(b-y_t)\,\dby t + \sigma\m y_t\circ \dby W_t,\\[3pt]
y(0) & = y_0\in\R_{\geq 0}\m,\nonumber
\end{align}
where $a,b\geq 0$ are the mean reversion parameters, $\sigma\geq 0$ is the volatility and $W$ denotes a standard real-valued Brownian motion. The $\circ$ means that this SDE is understood in the Stratonovich sense.
The SDE (\ref{eq:IGBM}) is known in the literature as Inhomogeneous Geometric Brownian Motion (or IGBM).
Using the control path $X = \{(t, W_t)\}_{t\geq 0}$ and setting $N = 3$, the log-ODE (\ref{eq:first_logode_def}) becomes
\begin{align*}
\frac{\dby z}{\dby u} & = a(b-z_u)h + \sigma\m z_u W_{s,t} - ab\sigma A_{s,t} + ab\sigma^2 L_{s,t}^{(1)} + a^2b\sigma L_{s,t}^{(2)},\\[3pt]
z(0) & = Y_s.
\end{align*}
where $h := t- s$ denotes the step size and the random variables $A_{s,t}, L_{s,t}^{(1)}, L_{s,t}^{(2)}$ are given by
\begin{align*}
A_{s,t} & := \int_s^t W_{s,r}\,\dby r - \frac{1}{2}hW_{s,t},\\[3pt]
L_{s,t}^{(1)} & := \int_s^t\int_s^r W_{s,v}\,\circ \dby W_v\,\dby r - \frac{1}{2}W_{s,t}A_{s,t} - \frac{1}{6}hW_{s,t}^2,\\[3pt]
L_{s,t}^{(2)} & := \int_s^t\int_s^r W_{s,v}\,\dby v\,\dby r - \frac{1}{2}h A_{s,t} - \frac{1}{6}h^2W_{s,t}.
\end{align*}
In \citet{foster2020poly}, the depth-3 log-signature of $X = \{(t, W_t)\}_{t\geq 0}$ was approximated so that the above log-ODE method became practical and this numerical scheme exhibited state-of-the-art convergence rates. For example, the approximation error produced by 25 steps of the
high order log-ODE method was similar to the error of the ``increment only'' log-ODE method with 1000 steps.
\end{example}

\section{Convergence of the log-ODE method for rough differential equations}\label{apx:logodeconv}

In this section, we shall present ``rough path'' error estimates for the log-ODE method. In addition, we will discuss the case when the vector fields governing the rough differential equation are linear.
We begin by stating the main result of \citet{logode2014estimate} which quantifies the approximation error of the log-ODE method in terms of the regularity of the systems vector field $f$ and control path $X$.
Since this section uses a number of technical definitions from rough path theory, we recommend \citet{roughpath2007notes} as an introduction to the subject. 

For $T > 0$, we will use the notation $\triangle_T := \{(s,t)\in [0,T]^2: s < t\}$ to denote a rescaled $2$-simplex.

\begin{theorem}[Lemma 15 in \citet{logode2014estimate}]\label{thm:logODEthm}
Consider the rough differential equation
\begin{align}
\dby Y_t & = f(Y_t)\,\dby X_t,\label{eq:RDE}\\
Y_0 & = \xi,\nonumber
\end{align}
where we make the following assumptions:
\begin{itemize}
\item $X$ is a \emph{geometric $p$-rough path} in $\R^d$, that is $X : \triangle_T \rightarrow T^{\floor{p}}(\R^d)$ is a continuous path in the tensor algebra
$T^{\floor{p}}(\R^d) := \R \oplus \R^d \oplus \big(\R^d\big)^{\otimes 2} \oplus \cdots \oplus \big(\R^d\big)^{\otimes \floor{p}}$ with increments
\begin{align}
X_{s,t} & = \Big(1, X_{s,t}^{(1)}, X_{s,t}^{(2)}, \cdots, X_{s,t}^{(\floor{p})}\Big),\label{eq:roughpathincrements}\\
X_{s,t}^{(k)} & := \pi_k\big(X_{s,t}\big),\nonumber
\end{align}
where $\pi_k : T^{\floor{p}}\big(\R^d\big)\rightarrow \big(\R^d\big)^{\otimes k}$ is the projection map onto $\big(\R^d\big)^{\otimes k}$, such that there exists a sequence of
continuous finite variation paths $x_n : [0,T] \rightarrow \R^d$ whose truncated signatures converge to $X$ in the \emph{$\boldsymbol{p}$-variation metric}:
\begin{equation}
d_p\Big(S^{\floor{p}}(x_n), X\Big) \rightarrow 0,
\label{eq:rpconvege}
\end{equation}
as $n\rightarrow\infty$, where the $p$-variation between two continuous paths $Z^1$ and $Z^2$ in $T^{\floor{p}}(\R^d)$ is
\begin{equation}
d_p\big(Z^1, Z^2\big) := \max_{1\leq k\leq \floor{p}}\sup_{\D}\bigg(\sum_{t_i\in\D}\Big\|\pi_k\big(Z_{t_i, t_{i+1}}^1\big) - \pi_k\big(Z_{t_i, t_{i+1}}^2\big)\Big\|^\frac{p}{k}\bigg)^\frac{k}{p},
\label{eq:rpmetric}
\end{equation}
where the supremum is taken over all partitions $\D$ of $[0,T]$ and the norms $\|\cdot\|$ must satisfy (up to some constant)
\begin{equation}
\|a\otimes b\| \leq \|a\|\|b\|,\nonumber
\end{equation}
for $a\in(\R^d)^{\otimes n}$ and $b\in(\R^d)^{\otimes m}$. For example, we can take $\|\cdot\|$ to be the projective or injective tensor norms (see Propositions 2.1 and 3.1 in \citet{tensorproducts2002book}).

\item The solution $Y$ and its initial value $\xi$ both take their values in $\R^n$.

\item The collection of vector fields $\{f_1, \cdots, f_d\}$ on $\R^n$ are denoted by $f : \R^n\rightarrow L(\R^n, \R^d)$,
where $L(\R^n, \R^d)$ is the space of linear maps from $\R^n$ to $\R^d$. We will assume that $f$ has \emph{\textup{Lip(}$\boldsymbol{\gamma}$\textup{)} regularity} with $\gamma > p$. 
That is, $f$ it is bounded with $\floor{\gamma}$ bounded derivatives, the last being H\"{o}lder continuous with exponent $(\gamma - \floor\gamma)$. Hence the following norm is finite:
\begin{equation}
\|f\|_{\mathrm{Lip}(\gamma)} := \max_{0 \leq k\leq \floor{\gamma}}\big\|D^k f\big\|_{\infty} \vee \big\|D^{\floor{\gamma}} f\big\|_{(\gamma - \floor{\gamma})-\text{H\"{o}l}}\,,
\label{eq:lipgamma}
\end{equation}
where $D^k f$ is the $k$-th (Fr\'{e}chet) derivative of $f$ and $\|\cdot\|_{\alpha\text{-H\"{o}l}}$ is the standard $\alpha$-H\"{o}lder norm with $\alpha\in(0,1)$.

\item The RDE (\ref{eq:RDE}) is defined in the Lyon's sense. Therefore by the Universal Limit Theorem
(see Theorem 5.3 in \citet{roughpath2007notes}), there exists a unique solution $Y : [0,T]\rightarrow\R^n$.
\end{itemize}

We define the log-ODE for approximating the solution $Y$ over an interval $[s,t]\subset [0,T]$ as follows: 
\begin{enumerate}
\item Compute the depth-$\floor{\gamma}$ log-signature of the control path $X$ over $[s,t]$. That is, we obtain $\logsig_{s,t}^{\floor{\gamma}}(X) := \log_{\floor{\gamma}}\big(S_{s,t}^{\floor{\gamma}}(X)\big) \in T^{\floor{\gamma}}(\R^d)$, where $\log_{\floor{\gamma}}(\cdot)$ is defined by projecting the standard tensor logarithm map onto $\{a\in T^{\floor{\gamma}}(\R^d) : \pi_0(a)>0\}$.

\item Construct the following (well-posed) ODE on the interval $[0,1]$,
\begin{align}
\frac{\dby z^{s,t}}{\dby u} & = F\big(z^{s,t}\big),\label{eq:standardlogode}\\
z_0^{s,t} & = Y_s,\nonumber
\end{align}
where the vector field $F:\R^n\rightarrow\R^n$ is defined from the log-signature as
\begin{equation}
F(z) := \sum_{k=1}^{\floor{\gamma}}f^{\circ k}(z)\pi_k\Big(\logsig_{s,t}^{\floor{\gamma}}(X)\Big).
\label{eq:logodevectfield}
\end{equation}
Recall that $f^{\circ k} : \R^n\rightarrow L((\R^d)^{\otimes k}, \R^n)$ was defined previously in Definition \ref{def:vect_derivative}.
\end{enumerate}
Then we can approximate $Y_t$ using the $u = 1$ solution of (\ref{eq:standardlogode}). Moreover, there exists a universal constant $C_{p,\gamma}$ depending only on $p$ and $\gamma$ such that 
\begin{equation}
\big\|Y_t - z_1^{s,t}\big\| \leq C_{p,\gamma}\|f\|_{\mathrm{Lip}(\gamma)}^\gamma\|X\|_{p\text{-var};[s,t]}^\gamma,
\label{eq:local_logodeestimate}
\end{equation}
where $\|\cdot\|_{p\text{-var};[s,t]}$ is the $p$-variation norm defined for paths in $T^{\floor{p}}(\R^d)$ by
\begin{equation}
\|X\|_{p\text{-var};[s,t]} := \max_{1\leq k\leq \floor{p}}\sup_{\D}\bigg(\sum_{t_i\in\D}\big\|X_{t_i, t_{i+1}}^k\big\|^\frac{p}{k}\bigg)^\frac{k}{p},
\label{eq:rpnorm}
\end{equation}
with the supremum taken over all partitions $\D$ of $[s,t]$.
\end{theorem}\medbreak
\begin{remark}
If the vector fields $\{f_1, \cdots, f_d\}$ are linear, then it immediately follows that $F$ is linear.
\end{remark}

Although the above theorem requires some sophisticated theory, it has a simple conclusion - namely
that log-ODEs can approximate controlled differential equations. That said, the estimate (\ref{eq:local_logodeestimate}) does not directly apply when the vector fields $\{f_i\}$ are linear as they would be unbounded. Fortunately, it is well known that linear RDEs are well posed and the growth of their solutions can be estimated.
\begin{theorem}[Theorem 10.57 in \citet{friz2010multidimensional}]\label{thm:linearexistance}
Consider the linear RDE on $[0,T]$
\begin{align*}
\dby Y_t & = f(Y_t)\,\dby X_t,\\
Y_0 & = \xi,
\end{align*}
where $X$ is a geometric $p$-rough path in $\R^d$, $\xi\in\R^n$ and the vector fields $\{f_i\}_{1\leq i\leq d}$ take the form $f_i(y) = A_i y + B$ where $\{A_{i}\}$ and $\{B_i\}$ are $n\times n$ matrices. Let $K$ denote an upper bound on $\max_i (\|A_i\| + \|B_i\|)$. Then a unique solution $Y:[0,T]\rightarrow\R^n$ exists. Moreover, it is bounded and there exists a constant $C_p$ depending only on $p$ such that
\begin{equation}
\|Y_t - Y_s\| \leq C_p\big(1+\|\xi\|\big)K\|X\|_{p\text{-var};[s,t]}\exp\Big(C_p K^p \|X\|_{p\text{-var};[s,t]}^p\Big),
\label{eq:linearRDEbound}
\end{equation}
for all $0\leq s\leq t\leq T$.
\end{theorem}

When the vector fields of the RDE (\ref{eq:RDE}) are linear, then the log-ODE (\ref{eq:standardlogode}) also becomes linear. Therefore, the log-ODE solution exists and is explicitly given as the exponential of the matrix $F$.

\begin{theorem}
Consider the same linear RDE on $[0,T]$ as in Theorem \ref{thm:linearexistance},
\begin{align*}
\dby Y_t & = f(Y_t)\,\dby X_t,\\
Y_0 & = \xi.\nonumber
\end{align*}
Then the log-ODE vector field $F$ given by (\ref{eq:logodevectfield}) is linear and the solution of the associated ODE (\ref{eq:standardlogode}) exists and satisfies
\begin{equation}
\|z_u^{s,t}\| \leq \|Y_s\|\exp\bigg(\sum_{m=1}^{\floor{\gamma}}K^m \Big\|\pi_m\Big(\logsig_{s,t}^{\floor{\gamma}}(X)\Big)\Big\|\bigg),
\label{eq:linearODEbound}
\end{equation}
for $u\in[0,1]$ and all $0\leq s\leq t\leq T$.
\end{theorem}
\begin{proof}
Since $F$ is a linear vector field on $\R^n$, we can view it as an $n\times n$ matrix and so for $u\in[0,1]$,
\begin{equation}
z_u^{s,t} = \exp(uF)z_0^{s,t},\nonumber
\end{equation}

where $\exp$ denotes the matrix exponential. The result now follows by the standard estimate $\|\exp(F)\| \leq \exp(\|F\|)$.
\end{proof}
\begin{remark}\label{rmk:linear_rmk}
Due to the boundedness of linear RDEs (\ref{eq:linearRDEbound}) and log-ODEs (\ref{eq:linearODEbound}), the arguments that established Theorem \ref{thm:logODEthm} will hold in the linear setting as $\|f\|_{\mathrm{Lip}(\gamma)}$ would be finite when defined on the domains that the solutions $Y$ and $z$ lie in.
\end{remark}

Given the local error estimate (\ref{eq:local_logodeestimate}) for the log-ODE method, we can now consider the approximation error that is exhibited by a log-ODE numerical solution to the RDE (\ref{eq:RDE}). Thankfully, the analysis required to derive such global error estimates was developed by Greg Gyurk\'{o} in his PhD thesis. 
Thus the following result is a straightforward application of Theorem 3.2.1 from \citet{gyurko2008thesis}.\medbreak

\begin{theorem} Let $X$, $f$ and $Y$ satisfy the assumptions given by Theorem \ref{thm:logODEthm} and suppose that $\{0 = t_0 < t_1 < \cdots < t_N = T\}$ is a partition of $[0,T]$ with $\max_{\,k}\|X\|_{p\text{-var};[t_k,t_{k+1}]}$ sufficiently small. We can construct a numerical solution $\{Y_k^{\log}\}_{0\leq k \leq N}$ of (\ref{eq:RDE}) by setting $Y_0^{\log} := Y_0$ and for each $k \in \{0, 1, \cdots, N - 1\}$, defining $Y_{k+1}^{\log}$ to be the solution at $u=1$ of the following ODE:
\begin{align}
\frac{\dby z^{t_k,t_{k+1}}}{\dby u} & := F\big(z^{t_k,t_{k+1}}\big),\label{eq:standardlogode2}\\
z_0^{t_k,t_{k+1}} & := Y_k^{\log},\nonumber
\end{align}
where the vector field $F$ is constructed from the log-signature of $X$ over the interval $[t_k, t_{k+1}]$ according to (\ref{eq:logodevectfield}). Then there exists a constant $C$ depending only on $p$, $\gamma$ and $\|f\|_{\mathrm{Lip}(\gamma)}$ such that
\begin{equation}
\big\|Y_{t_k} - Y_k^{\log}\big\| \leq C\sum_{i=0}^{k-1}\|X\|_{p\text{-var};[t_i,t_{i+1}]}^\gamma,
\label{eq:global_logodeestimate}
\end{equation}
for $0\leq k\leq N$.
\end{theorem}
\begin{remark}
The above error estimate also holds when the vector field $f$ is linear (by Remark \ref{rmk:linear_rmk})).
\end{remark}

Since $\floor{\gamma}$ is the truncation depth of the log-signatures used to construct each log-ODE vector field, we see that high convergence rates can be achieved through using more terms in each log-signature.
It is also unsurprising that the error estimate (\ref{eq:global_logodeestimate}) increases with the ``roughness'' of the control path.
So just as in our experiments, we see that the performance of the log-ODE method can be improved by choosing an appropriate step size and depth of log-signature.

\section{Experimental details} \label{apx:experiments}

\paragraph{Code} The code to reproduce the experiments is available at \url{https://github.com/jambo6/neuralRDEs}. 

\paragraph{Data splits} Each dataset was split into a training, validation, and testing dataset with relative sizes 70\%/15\%/15\%.

\paragraph{Hyperparameter selection} Hyperparameters were selected for the Neural CDE model by performing a grid search, with a step size chosen so that the length of the sequence was 500 steps. This was found to create a reasonable balance between training time and sequence length. We additionally performed a separate hyperparameter selection for the ODE-RNN model. The Neural RDE models then use the same hyperparameters as the Neural CDE model.

\paragraph{Normalisation} The training splits of each dataset were normalised to zero mean and unit variance. The statistics from the training set were then used to normalise the validation and testing datasets.

\paragraph{Architecture} We give a graphical description of the architecture used for updating the Neural CDE hidden state in figure \ref{fig:network_diagram}. The input is first run through a multilayer perceptron with $n$ layers of size $h$, with with $n, h$ being hyperparameters. ReLU nonlinearities are used at each layer except the final one, where we instead use a tanh nonlinearity. The goal of this is to help prevent term blow-up over the long sequences.

Note that this is a small inconsistency between this work and the original model proposed in \citet{kidger2020neural}. Here, we applied the tanh function as the final hidden layer nonlinearity, whilst in the original paper the tanh nonlinearity is applied after the final linear map. Both methods are used to constrain the rate of change of the hidden state; we do not know of a reason to prefer one over the other.

Note that the final linear layer in the multilayer perceptron is reshaped to produce a matrix-valued output, of shape $v \times p$. (As $\widehat{f}_\theta$ is matrix-valued.) A matrix-vector multiplication with the log-signature then produces the vector field for the ODE solver.

\begin{figure*}
    \centering
    \resizebox{\textwidth}{!}{%
    
    \begin{tikzpicture}
        \node[input, rotate=90, minimum width=2.8cm] (input) at (0, 0) {Input, $Z_{r_i}$};
        \node[hidden, rotate=90, minimum width=5cm] (hidden1) at (2,0) {Hidden layer 1};
        \node[hidden, rotate=90, minimum width=5cm] (hidden_end) at (6,0) {Hidden layer n};
        \node[hidden_square, minimum width=4cm] (hidden_square) at (10,0) {$\hat{f}_\theta(Z_{r_i})$};
        \node[logsig, rotate=90, minimum width=3.5cm] (logsig) at (12.6, 0) {$\logsig_{r_i, r_{i+1}}$};
        \node[input, rotate=90, minimum width=2.8cm] (output) at (15.5, 0) {Output, $Z_{r_{i+1}}$};
        
        \node[below, rotate=90] at (input.north east) [xshift=-2.4ex, yshift=0.4ex] {\footnotesize $v \times 1$};
        \node[below, rotate=90] at (hidden1.north east) [xshift=-2.4ex, yshift=0.4ex] {\footnotesize $h \times 1$};
        \node[below, rotate=90] at (hidden_end.north east) [xshift=-2.4ex, yshift=0.4ex] {\footnotesize $h \times 1$};
        \node[below, rotate=0] at (hidden_square.north east) [xshift=-2.4ex, yshift=0.4ex] {\footnotesize $v \times p$};
        \node[below, rotate=90] at (logsig.north east) [xshift=-2.4ex, yshift=0.4ex] {\footnotesize $p \times 1$};
        \node[below, rotate=90] at (output.north east) [xshift=-2.4ex, yshift=0.4ex] {\footnotesize $v \times 1$};

        \node (dots) at (4, 0) {\ldots};
        
        \draw[middlearrow={>}] (12.6, -3) -- (logsig) node[midway, above, rotate=90] {Logsig} node[midway, below, rotate=90] {factor}; 
        
        \draw[middlearrow={>}] (input) -- (hidden1) node[midway, above] {ReLU};
        \draw[middlearrow={>}] (hidden1) -- (dots) node[midway, above] {ReLU};
        \draw[middlearrow={>}] (dots) -- (hidden_end) node[midway, above] {Tanh};
        \draw[middlearrow={>}] (hidden_end) -- (hidden_square) node[midway, above] {Linear} node[midway, below] {+ reshape};
        \draw[middlearrow={>}] (logsig) -- (output) node[midway, above] {ODE Solve};
    
        \draw[decoration={brace,raise=5pt},decorate] (8, 1.8) -- node[above=6pt] {Matrix multiplication} (13, 1.8);
        \draw[decoration={brace,raise=5pt},decorate] (2, 2.5) -- node[above=6pt] {n layers} (6, 2.5);
        \draw[decoration={brace,mirror,raise=5pt},decorate] (0.5, -2.5) -- node[below=6pt] {$\hat{f}_\theta$} (8, -2.5);

    \end{tikzpicture}
}
    \caption{Overview of the hidden state update network structure. We give the dimensions at each layer in the top right hand corner of each box.} 
    \label{fig:network_diagram}
\end{figure*}
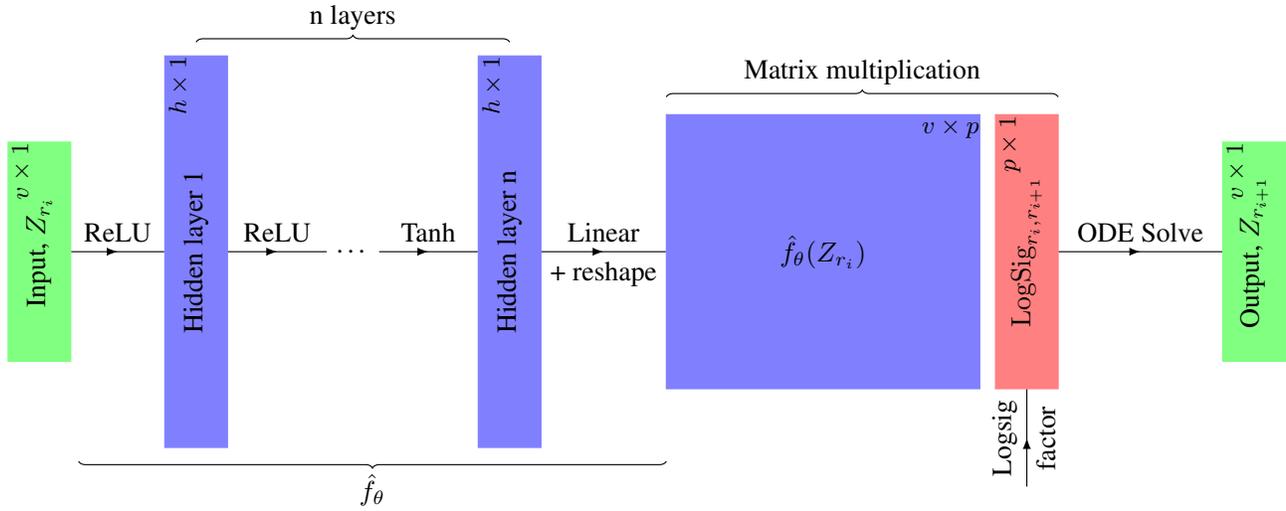

\paragraph{ODE Solver} All problems used the `rk4' solver as implemented by \texttt{torchdiffeq} \citep{torchdiffeq} version 0.0.1.

\paragraph{Computing infrastructure} All EigenWorms experiments were run on a computer equipped with three GeForce RTX 2080 Ti's. All BIDMC experiments were run on a computed with two GeForce RTX 2080 Ti's and two Quadro GP100's.

\paragraph{Optimiser} All experiments used the Adam optimiser. The learning rate was initialised at $0.032$ divided by batch size. The batch size used was 1024 for EigenWorms and 512 for the BIDMC problems. If the validation loss failed to decrease after 15 epochs the learning rate was reduced by a factor of 10. If the validation loss did not decrease after 60 epochs, training was terminated and the model was rolled back to the point at which it achieved the lowest loss on the validation set. 

\paragraph{Hyperparameter selection} Hyperparameters were selected to optimise the score of the NCDE$_1$ model on the validation set. For each dataset the search was performed with a step size that meant the total number of hidden state updates was equal to 500, as this represented a good balance between length and speed that allowed us to complete the search in a reasonable time-frame. In particular, this was short enough that we could train using the non-adjoint training method which helped to speed this section up. The hyperparameters that were considered were:
\begin{itemize}
    \item Hidden dimension: [16, 32, 64] - The dimension of the hidden state $Z_t$.
    \item Number of layers: [2, 3, 4] - The number of hidden state layers.
    \item Hidden hidden multiplier: [1, 2, 3] - Multiplication factor for the hidden hidden state, this being the `Hidden layer $k$' in figure \ref{fig:network_diagram}. The dimension of each of these `hidden hidden' layers with be this value multiplied by `Hidden dimension'.
\end{itemize}
We ran each of these 27 total combinations for every dataset and the parameters that corresponded were used as the parameters when training over the full depth and step grid. The full results from the hyperparameter search are listed in tables (\ref{tab:eigenworms_hyper}, \ref{tab:bidmc_hyper}) with bolded values to show which values were eventually selected.

\begin{table*}[t]
    \scriptsize
    \begin{center}
        \begin{tabular}{ccccc}
        \toprule
        \textbf{Validation accuracy} & \textbf{Hidden dim} & \textbf{Num layers} & \textbf{Hidden hidden multiplier} & \textbf{Total params} \\
        \midrule
        33.3 &         16 &                        2 &          3 &        5509 \\
        43.6 &         16 &                        2 &          2 &        5509 \\
        56.4 &         16 &                        2 &          1 &        4453 \\
        64.1 &         16 &                        3 &          2 &        8869 \\
        38.5 &         16 &                        3 &          3 &        8869 \\
        51.3 &         16 &                        3 &          1 &        6517 \\
        82.1 &         16 &                        4 &          2 &       12741 \\
        35.9 &         16 &                        4 &          3 &       12741 \\
        53.8 &         16 &                        4 &          1 &        8581 \\
        35.9 &         32 &                        2 &          3 &       21253 \\
        74.4 &         32 &                        2 &          2 &       21253 \\
        43.6 &         32 &                        2 &          1 &       17093 \\
        53.8 &         32 &                        3 &          3 &       34629 \\
        \textbf{87.2} &         \textbf{32} & \textbf{3} & \textbf{2} & \textbf{34629} \\
        64.1 &         32 &                        3 &          1 &       25317 \\
        35.9 &         32 &                        4 &          3 &       50053 \\
        71.8 &         32 &                        4 &          1 &       33541 \\
        79.5 &         32 &                        4 &          2 &       50053 \\
        41.0 &         64 &                        2 &          3 &       83461 \\
        64.1 &         64 &                        2 &          2 &       83461 \\
        48.7 &         64 &                        3 &          3 &      136837 \\
        59.0 &         64 &                        3 &          2 &      136837 \\
        51.3 &         64 &                        2 &          1 &       66949 \\
        56.4 &         64 &                        4 &          2 &      198405 \\
        64.1 &         64 &                        4 &          3 &      198405 \\
        64.1 &         64 &                        3 &          1 &       99781 \\
        51.3 &         64 &                        4 &          1 &      132613 \\
        \bottomrule
        \end{tabular}
    \end{center}
    \caption{Hyperparamter selection results for the EigenWorms dataset. The blue values denote the selected hyperparameters.}
    \label{tab:eigenworms_hyper}
\end{table*} 

\begin{table*}[t]
    \scriptsize
    \begin{center}
        \begin{tabular}{ccc}
        \toprule
        \textbf{Validation accuracy} & \textbf{Hidden dim} & \textbf{Total params} \\
        \midrule
        61.5 &               32 &       11299 \\
        53.8 &               64 &       24611 \\
        \textbf{64.1} &              \textbf{128} &       \textbf{57379} \\
        59.0 &              192 &       98339 \\
        61.5 &              256 &      147491 \\
        59.0 &              320 &      204835 \\
        64.1 &              388 &      274739 \\
        \bottomrule
        \end{tabular}
    \end{center}
    \caption{Hyperparamter selection results for the ODE-RNN model on the EigenWorms dataset}
    \label{tab:eigenworms_hyper_odernn}
\end{table*}

\begin{table*}[t]
    \scriptsize
    \begin{center}
        \begin{tabular}{ccccccc}
        \toprule
        \multicolumn{3}{c}{\textbf{Validation loss}} & \multirow{2}{*}{\textbf{Hidden dim}} & \multirow{2}{*}{\textbf{Num layers}} & \multirow{2}{*}{\textbf{Hidden hidden multiplier}} & \multirow{2}{*}{\textbf{Total params}} \\
        \cmidrule(lr){1-3}
        RR & HR & SpO2 & & & & \\
        \midrule
        1.72 &      6.10 &      2.07 &         16 &                        2 &          1 &        2209 \\
        1.57 &      5.58 &      1.97 &         16 &                        2 &          2 &        3265 \\
        1.55 &      6.10 &      1.33 &         16 &                        2 &          3 &        3265 \\
        1.80 &      5.16 &      2.05 &         16 &                        3 &          1 &        3249 \\
        1.61 &      5.22 &      1.62 &         16 &                        3 &          2 &        5601 \\
        1.56 &      3.34 &      1.18 &         16 &                        3 &          3 &        5601 \\
        1.57 &      3.86 &      1.97 &         16 &                        4 &          1 &        4289 \\
        1.45 &      3.54 &      1.25 &         16 &                        4 &          2 &        8449 \\
        1.54 &      3.93 &      1.09 &         16 &                        4 &          3 &        8449 \\
        1.56 &      6.81 &      1.87 &         32 &                        2 &          1 &        8513 \\
        1.42 &      3.11 &      1.43 &         32 &                        2 &          2 &       12673 \\
        1.54 &      3.60 &      1.11 &         32 &                        2 &          3 &       12673 \\
        1.54 &      3.52 &      1.57 &         32 &                        3 &          1 &       12641 \\
        1.39 &      2.96 &      1.03 &         32 &                        3 &          2 &       21953 \\
        1.47 &      2.95 &      1.05 &         32 &                        3 &          3 &       21953 \\
        1.55 &      3.00 &      2.00 &         32 &                        4 &          1 &       16769 \\
        1.38 &      3.20 &      1.07 &         32 &                        4 &          2 &       33281 \\
        1.43 &      2.58 &      1.01 &         32 &                        4 &          3 &       33281 \\
        1.51 &      3.21 &      1.10 &         64 &                        2 &          1 &       33409 \\
        1.43 &     \textbf{2.22} & 1.00 &         \textbf{64} &                        \textbf{2} &          \textbf{2} &       \textbf{49921} \\
        1.51 &      3.34 &      0.94 &         64 &                        2 &          3 &       49921 \\
        1.55 &      3.24 &      2.09 &         64 &                        3 &          1 &       49857 \\
        1.32 &      2.53 &      0.88 &         64 &                        3 &          2 &       86913 \\
        \textbf{1.25} &      2.57 &      \textbf{0.73} &         \textbf{64} &                        \textbf{3} &         \textbf{3} &       \textbf{86913} \\
        1.43 &      5.78 &      1.43 &         64 &                        4 &          1 &       66305 \\
        1.28 &      2.26 &      0.93 &         64 &                        4 &          2 &      132097 \\
        1.32 &      2.46 &      1.15 &         64 &                        4 &          3 &      132097 \\
        \bottomrule
        \end{tabular}
    \end{center}
    \caption{Hyperparameter selection results for each problem of the BIDMC dataset. The bold values denote the selected hyperparameters for each vitals sign problem. Note that RR and SpO2 had the same parameters selected, hence why only two lines are given in bold.}
    \label{tab:bidmc_hyper}
\end{table*}

\begin{table*}[t]
    \scriptsize
    \begin{center}
        \begin{tabular}{ccccc}
        \toprule
        \multicolumn{3}{c}{\textbf{Validation loss}} & \multirow{2}{*}{\textbf{Hidden dim}} & \multirow{2}{*}{\textbf{Total params}} \\
        \cmidrule(lr){1-3}
        RR & HR & SpO2 & & \\
        \midrule
        3.00 &     \textbf{12.82} &     \textbf{3.37} &               32 &        3871 \\
        3.00 &     12.82 &      3.37 &               64 &        9759 \\
        2.82 &     12.82 &      3.37 &              128 &       27679 \\
        \textbf{2.49} &     12.82 &      3.37 &              192 &       53791 \\
        2.52 &     12.82 &      3.37 &              256 &       88095 \\
        2.50 &     12.82 &      3.37 &              320 &      130591 \\
        2.83 &     12.82 &      3.37 &              388 &      184719 \\
        \bottomrule
        \end{tabular}
    \end{center}
    \caption{Hyperparameter selection results for the folded ODE-RNN model on the BIDMC problem. Bold values indicate selected hyperparamter values. The ODE-RNN model failed to train effectively for the HR and SpO2 problems which is why the validation losses are the same (to 2dp).}
    \label{tab:bidmc_hyper_odernn}
\end{table*}

\section{Experimental Results} \label{apx:results}
Here we include the full breakdown of all experimental results. Tables \ref{tab:eigenworms_all} and \ref{tab:bidmc_all} include all results from the EigenWorms and BIDMC datasets respectively.

\begin{table*}[t]
    \small
    \begin{center}
        \begin{tabular}{ccccc}
        \toprule
        \textbf{Model} & \textbf{Step} & \textbf{Test Accuracy} & \textbf{Time (Hrs)} & \textbf{Memory (Mb)} \\
        \midrule
          & 1    &  Memory Error & Memory Error & Memory Error \\
          & 2    &   36.8 $\pm$ 1.5 &           1.6 &        7170.1 \\
          & 4    &   35.0 $\pm$ 1.5 &           0.8 &        3629.3 \\
          & 6    &   36.8 $\pm$ 1.5 &           0.5 &        2448.6 \\
          & 8    &   36.8 $\pm$ 1.5 &           0.4 &        1858.8 \\
          & 16   &   32.5 $\pm$ 3.0 &           0.2 &         973.5 \\
        ODE-RNN  & 32   &   32.5 $\pm$ 1.5 &           0.1 &         532.2 \\
        (folded)  & 64   &   41.0 $\pm$ 4.4 &           0.1 &         311.2 \\
          & 128  &   47.9 $\pm$ 5.3 &           0.0 &         200.8 \\
          & 256  &   46.2 $\pm$ 0.0 &           0.0 &         147.0 \\
          & 512  &  47.9 $\pm$ 10.4 &           0.0 &         124.5 \\
          & 1024 &   44.4 $\pm$ 7.4 &           0.0 &         122.4 \\
          & 2048 &   48.7 $\pm$ 6.8 &           0.0 &         137.2 \\
        
        \midrule
          & 1    &  62.4 $\pm$ 12.1 &          22.0 &         176.5 \\
          & 2    &   69.2 $\pm$ 4.4 &          14.6 &          90.6 \\
          & 4    &  66.7 $\pm$ 11.8 &           5.5 &          46.6 \\
          & 6    &  65.8 $\pm$ 12.9 &           2.6 &          31.5 \\
          & 8    &  64.1 $\pm$ 13.3 &           3.1 &          24.3 \\
          & 16   &  64.1 $\pm$ 16.8 &           1.5 &          13.4 \\
        \multirow{2}{*}{NCDE}  & 32   &  64.1 $\pm$ 14.3 &           0.5 &           8.0 \\
          & 64   &   56.4 $\pm$ 6.8 &           0.4 &           5.2 \\
          & 128  &   48.7 $\pm$ 2.6 &           0.1 &           3.9 \\
          & 256  &   42.7 $\pm$ 3.0 &           0.1 &           3.2 \\
          & 512  &   44.4 $\pm$ 5.3 &           0.0 &           2.9 \\
          & 1024 &  41.9 $\pm$ 14.6 &           0.0 &           2.7 \\
          & 2048 &   38.5 $\pm$ 5.1 &           0.0 &           2.6 \\
          
        \midrule
          & 2    &  \textbf{76.1 $\pm$ 13.2} &           9.8 &         354.3 \\
          & 4    &   \textbf{83.8 $\pm$ 3.0} &           2.4 &         180.0 \\
          & 6    &   \textbf{76.9 $\pm$ 6.8} &           2.0 &          82.2 \\
          & 8    &   \textbf{77.8 $\pm$ 5.9} &           2.1 &          94.2 \\
          & 16   &   \textbf{78.6 $\pm$ 3.9} &           1.3 &          50.2 \\
        NRDE$_2$ & 32   &  67.5 $\pm$ 12.1 &           0.7 &          28.1 \\
          & 64   &   73.5 $\pm$ 7.8 &           0.4 &          17.2 \\
          & 128  &   \textbf{76.1 $\pm$ 5.9} &           0.2 &           7.8 \\
          & 256  &  \textbf{72.6 $\pm$ 12.1} &           0.1 &           8.9 \\
          & 512  &  \textbf{69.2 $\pm$ 11.8} &           0.0 &           7.6 \\
          & 1024 &   \textbf{65.0 $\pm$ 7.4} &           0.0 &           6.9 \\
          & 2048 &   \textbf{67.5 $\pm$ 3.9} &           0.0 &           6.5 \\
          
        \hdashline\noalign{\vskip 0.5ex}
          & 2    &   66.7 $\pm$ 4.4 &           7.4 &        1766.2 \\
          & 4    &   76.9 $\pm$ 9.2 &           2.8 &         856.8 \\
          & 6    &   70.9 $\pm$ 1.5 &           1.4 &         606.1 \\
          & 8    &   70.1 $\pm$ 6.5 &           1.3 &         460.7 \\
          & 16   &   73.5 $\pm$ 3.0 &           1.4 &         243.7 \\
        NRDE$_3$ & 32   &   \textbf{75.2 $\pm$ 3.0} &           0.6 &         134.7 \\
          & 64   &  \textbf{74.4 $\pm$ 11.8} &           0.3 &          81.0 \\
          & 128  &   68.4 $\pm$ 8.2 &           0.1 &          53.3 \\
          & 256  &   60.7 $\pm$ 8.2 &           0.1 &          40.2 \\
          & 512  &  62.4 $\pm$ 10.4 &           0.0 &          33.1 \\
          & 1024 &   59.8 $\pm$ 3.9 &           0.0 &          29.6 \\
          & 2048 &   61.5 $\pm$ 4.4 &           0.0 &          27.7 \\
        \bottomrule
        \end{tabular}
    \end{center}
    \caption{Mean and standard deviation of test set accuracy (in \%) over three repeats, as well as memory usage and training time, on the EigenWorms dataset for depths 1--3 and a small selection of step sizes. The bold values denote that the model was the top performer for that step size.}
    \label{tab:eigenworms_all}
\end{table*}

\begin{table*}[t]
    \small
    \begin{center}
        \begin{tabular}{ccccccccc}
        \toprule
        \multirow{2}{*}{\textbf{Depth}} & \multirow{2}{*}{\textbf{Step}} & \multicolumn{3}{c}{$\mathbf{L^2}$} & \multicolumn{3}{c}{\textbf{Time (H)}} & \multirow{2}{*}{\textbf{Memory (Mb)}} \\
        \cmidrule(lr){3-5} \cmidrule(lr){6-8}
        & & RR & HR & SpO$_2$ & RR & HR & SpO$_2$ & \\
        \midrule
         & 1    &    Error &  13.06 $\pm$ 0.0 & Error &           Error &          10.4 &           Error &             3654.0 \\
          & 2    &    Error &  13.06 $\pm$ 0.0 &  Error &           Error &           5.5 &           Error &             1840.4 \\
          & 4    &  2.76 $\pm$ 0.14 &  13.06 $\pm$ 0.0 &    3.3 $\pm$ 0.0 &           3.0 &           2.7 &           2.1 &             1809.0 \\
          & 8    &  2.47 $\pm$ 0.35 &  13.06 $\pm$ 0.0 &    3.3 $\pm$ 0.0 &           1.5 &           1.2 &           0.9 &              917.2 \\
          & 16   &  2.21 $\pm$ 0.75 &  13.06 $\pm$ 0.0 &    3.3 $\pm$ 0.0 &           2.2 &           0.7 &           0.4 &              471.9 \\
        ODE-RNN   & 32   &  1.82 $\pm$ 0.64 &  13.06 $\pm$ 0.0 &    3.3 $\pm$ 0.0 &           0.7 &           0.3 &           0.2 &              249.4 \\
        (folded)  & 64   &   1.6 $\pm$ 0.22 &  13.06 $\pm$ 0.0 &    3.3 $\pm$ 0.0 &           0.5 &           0.1 &           0.1 &              137.0 \\
          & 128  &  1.62 $\pm$ 0.07 &  13.06 $\pm$ 0.0 &    3.3 $\pm$ 0.0 &           0.2 &           0.1 &           0.1 &               81.9 \\
          & 256  &  1.57 $\pm$ 0.04 &  7.04 $\pm$ 1.04 &  \textbf{1.43 $\pm$ 0.11} &           0.1 &           0.1 &           0.1 &               53.8 \\
          & 512  &  1.66 $\pm$ 0.06 &   6.75 $\pm$ 0.9 &  1.98 $\pm$ 0.31 &           0.0 &           0.1 &           0.1 &               40.4 \\
          & 1024 &  \textbf{1.69 $\pm$ 0.02} &   8.4 $\pm$ 0.28 &  2.05 $\pm$ 0.14 &           0.0 &           0.0 &           0.0 &               36.2 \\
          & 2048 &  1.75 $\pm$ 0.03 &   9.2 $\pm$ 0.27 &  \textbf{2.24 $\pm$ 0.11} &           0.0 &           0.0 &           0.0 &               39.6 \\
        \midrule
          & 1    &  2.79 $\pm$ 0.04 &   9.82 $\pm$ 0.34 &  2.83 $\pm$ 0.27 &          23.8 &          22.1 &          28.1 &               56.5 \\
          & 2    &  2.87 $\pm$ 0.03 &  11.69 $\pm$ 0.38 &   \textbf{3.36 $\pm$ 0.2} &          19.3 &           9.6 &           8.8 &               32.6 \\
          & 4    &  \textbf{2.92 $\pm$ 0.08} &  11.15 $\pm$ 0.49 &  3.69 $\pm$ 0.06 &           5.3 &           5.7 &           3.2 &               20.2 \\
          & 8    &   2.8 $\pm$ 0.06 &  10.72 $\pm$ 0.24 &  3.43 $\pm$ 0.17 &           3.0 &           2.6 &           4.8 &               14.3 \\
          & 16   &  2.22 $\pm$ 0.07 &   7.98 $\pm$ 0.61 &   2.9 $\pm$ 0.11 &           1.7 &           1.4 &           1.8 &               11.8 \\
        \multirow{2}{*}{NCDE}  & 32   &  2.53 $\pm$ 0.23 &  12.23 $\pm$ 0.43 &  2.68 $\pm$ 0.12 &           1.9 &           0.9 &           2.2 &                9.8 \\
          & 64   &  2.63 $\pm$ 0.11 &  12.02 $\pm$ 0.09 &  2.88 $\pm$ 0.06 &           0.2 &           0.3 &           0.4 &                9.1 \\
          & 128  &  2.64 $\pm$ 0.18 &  11.98 $\pm$ 0.37 &  2.86 $\pm$ 0.04 &           0.2 &           0.2 &           0.3 &                8.7 \\
          & 256  &  2.53 $\pm$ 0.04 &   12.29 $\pm$ 0.1 &   3.08 $\pm$ 0.1 &           0.1 &           0.1 &           0.1 &                8.3 \\
          & 512  &  2.53 $\pm$ 0.03 &  12.22 $\pm$ 0.11 &  2.98 $\pm$ 0.04 &           0.1 &           0.0 &           0.1 &                8.4 \\
          & 1024 &  2.67 $\pm$ 0.12 &  11.55 $\pm$ 0.03 &  2.91 $\pm$ 0.12 &           0.1 &           0.1 &           0.1 &                8.4 \\
          & 2048 &  2.48 $\pm$ 0.03 &   12.03 $\pm$ 0.2 &  3.25 $\pm$ 0.01 &           0.0 &           0.1 &           0.0 &                8.2 \\
        \midrule
          & 2    &   2.91 $\pm$ 0.1 &  11.11 $\pm$ 0.23 &  3.89 $\pm$ 0.44 &          12.7 &           9.3 &           8.2 &               58.3 \\
          & 4    &  \textbf{2.92 $\pm$ 0.04} &   11.14 $\pm$ 0.2 &  4.23 $\pm$ 0.57 &          18.1 &           5.0 &           3.4 &               34.0 \\
          & 8    &  2.63 $\pm$ 0.12 &   8.63 $\pm$ 0.24 &  2.88 $\pm$ 0.15 &           2.1 &           3.4 &           3.3 &               21.8 \\
          & 16   &   1.8 $\pm$ 0.07 &   5.73 $\pm$ 0.45 &  1.98 $\pm$ 0.21 &           2.2 &           1.4 &           2.5 &               16.0 \\
          & 32   &   1.9 $\pm$ 0.02 &     7.9 $\pm$ 1.0 &   1.69 $\pm$ 0.2 &           1.2 &           1.1 &           2.0 &               13.1 \\
        NRDE$_2$  & 64   &  1.89 $\pm$ 0.04 &   5.54 $\pm$ 0.45 &  2.04 $\pm$ 0.07 &           0.3 &           0.3 &           1.7 &               11.6 \\
          & 128  &  1.86 $\pm$ 0.03 &   6.77 $\pm$ 0.42 &  1.95 $\pm$ 0.18 &           0.3 &           0.4 &           0.7 &               10.9 \\
          & 256  &  1.86 $\pm$ 0.09 &   5.64 $\pm$ 0.19 &   2.1 $\pm$ 0.19 &           0.1 &           0.1 &           0.5 &               10.5 \\
          & 512  &  1.81 $\pm$ 0.02 &   5.05 $\pm$ 0.23 &  2.17 $\pm$ 0.18 &           0.1 &           0.2 &           0.4 &               10.3 \\
          & 1024 &  1.93 $\pm$ 0.11 &    6.0 $\pm$ 0.19 &  2.41 $\pm$ 0.07 &           0.1 &           0.1 &           0.2 &               10.2 \\
          & 2048 &  \textbf{2.03 $\pm$ 0.03} &    \textbf{7.7 $\pm$ 1.46} &  2.55 $\pm$ 0.03 &           0.1 &           0.1 &           0.1 &               10.2 \\
        \hdashline\noalign{\vskip 0.5ex}
          & 2    &  \textbf{2.82 $\pm$ 0.08} &  \textbf{11.01 $\pm$ 0.28} &   4.1 $\pm$ 0.72 &           8.8 &           9.4 &           6.9 &              125.2 \\
          & 4    &  2.97 $\pm$ 0.23 &  \textbf{10.13 $\pm$ 0.62} &  \textbf{3.56 $\pm$ 0.44} &           3.2 &           4.1 &           2.6 &               71.6 \\
          & 8    &  \textbf{2.42 $\pm$ 0.19} &    \textbf{7.67 $\pm$ 0.4} &  \textbf{2.55 $\pm$ 0.13} &           2.9 &           3.2 &           3.1 &               43.3 \\
          & 16   &  \textbf{1.74 $\pm$ 0.05} &   \textbf{4.11 $\pm$ 0.61} &   \textbf{1.4 $\pm$ 0.06} &           1.4 &           1.4 &           6.5 &               29.1 \\
          & 32   & \textbf{1.67 $\pm$ 0.01} &     \textbf{4.5 $\pm$ 0.7} &  \textbf{1.61 $\pm$ 0.05} &           1.3 &           1.8 &           7.3 &               20.5 \\
        NRDE$_3$  & 64   &  \textbf{1.53 $\pm$ 0.08} &   \textbf{3.05 $\pm$ 0.36} &  \textbf{1.48 $\pm$ 0.14} &           0.4 &           1.9 &           3.3 &               17.9 \\
          & 128  &  \textbf{1.51 $\pm$ 0.08} &   \textbf{2.97 $\pm$ 0.45}$^*$ &  \textbf{1.37 $\pm$ 0.22} &           0.5 &           1.7 &           1.7 &               17.3 \\
          & 256  &  \textbf{1.51 $\pm$ 0.06} &    \textbf{3.4 $\pm$ 0.74} &  1.47 $\pm$ 0.07 &           0.3 &           0.7 &           0.6 &               16.6 \\
          & 512  &  \textbf{1.49 $\pm$ 0.08}$^*$ &   \textbf{3.46 $\pm$ 0.13} &  \textbf{1.29 $\pm$ 0.15}$^*$ &           0.3 &           0.4 &           0.4 &               15.4 \\
          & 1024 &  1.83 $\pm$ 0.33 &    \textbf{5.58 $\pm$ 2.5} &  \textbf{1.72 $\pm$ 0.31} &           0.2 &           0.1 &           0.1 &               15.7 \\
          & 2048 &  2.31 $\pm$ 0.27 &   9.77 $\pm$ 1.53 &  2.45 $\pm$ 0.18 &           0.1 &           0.1 &           0.1 &               15.6 \\
        \bottomrule
        \end{tabular}
    \end{center}
    \caption{Mean and standard deviation of the $L^2$ losses on the test set for each of the vitals signs prediction tasks (RR, HR, SpO$_2$) on the BIDMC dataset, across three repeats. Only mean times are shown for space. The memory usage is given as the mean over all three of the tasks as it was approximately the same for any task for a given depth and step. Error denotes that the model could not be run within GPU memory. The bold values denote the algorithm with the lowest test set loss for a fixed step size for each task.}
    \label{tab:bidmc_all}
\end{table*}

\end{document}